\documentclass{article}

\usepackage{ifthen}
\newboolean{ARXIV}  
\setboolean{ARXIV}{False}

\ifthenelse{\boolean{ARXIV}}{
\usepackage[preprint, nonatbib]{neurips_2020}  
}{
\usepackage[final,nonatbib]{neurips_2020}
}

\usepackage[utf8]{inputenc} 
\usepackage[T1]{fontenc}    
\usepackage{hyperref}       
\usepackage{url}            
\usepackage{booktabs}       
\usepackage{amsfonts}       
\usepackage{nicefrac}       
\usepackage{microtype}      
\usepackage{adjustbox}

\usepackage{amsmath,amsthm,amsfonts,amssymb}
\usepackage{mathtools}
\usepackage{bm, bbm}
\usepackage{xcolor, graphicx}
\usepackage{subcaption}
\usepackage{cleveref}
\usepackage[ruled,vlined]{algorithm2e}
\usepackage{todonotes}

\crefformat{section}{\S#2#1#3}

\newcommand{\ti}{\textit}
\newcommand{\tb}{\textbf}

\newcommand{\eg}{\ti{e.g.}}
\newcommand{\ie}{\ti{i.e.}}
\newcommand{\etal}{{\it et al.}}

\newtheorem{lemmasupp}{Lemma}[section]
\newtheorem{proposition}{Proposition}
\newtheorem{propositionsupp}{Proposition}[section]

\newcommand{\reals}{\mathbb{R}}

\newcommand{\E}{\mathbb{E}}

\newcommand{\Var}{\mathrm{Var}}
\DeclarePairedDelimiterX{\infdivx}[2]{(}{)}{
  #1\;\delimsize|\delimsize|\;#2
}
\newcommand{\kld}[2]{\ensuremath{D_\text{KL}\infdivx{#1}{#2}}} 

\DeclareMathOperator*{\argmax}{arg\,max}

\DeclareMathOperator*{\argmin}{arg\,min}
\DeclareMathOperator*{\median}{median}

\newcommand{\x}{\mathbf{x}}
\newcommand{\betap}{\beta}
\newcommand{\xsupi}{\mathbf{x}_i}
\newcommand{\ysupi}{y_i}
\newcommand{\trainset}{\{(\xsupi, \ysupi)\}_{i=1}^n}
\newcommand{\tildext}{\tilde{\x}^{(t)}}
\newcommand{\sampleset}{\{\tildext_i\}_{i = 1}^m}
\newcommand{\searchdist}{p_{\theta^{(t)}}(\x)}
\newcommand{\searchdistminus}{p_{\theta^{(t-1)}}(\x)}
\newcommand{\searchdistxsupi}{p_{\theta^{(t)}}(\xsupi)}
\newcommand{\searchdiszxsupi}{p_{\theta^{(0)}}(\xsupi)}

\newcommand{\searchdissub}{p_\theta(\x)}
\newcommand{\searchdistsub}{p_{\theta^{(t)}}(\x)}
\newcommand{\searchdistfinal}{p_{\theta^{(T)}}(\x)}
\newcommand{\searchdis}{\searchdissub}
\newcommand{\traindis}{p_0(\x)}
\newcommand{\fig}{Figure~}
\newcommand{\suppinfo}{Supplementary Material}

\newcommand{\sth}[1]{#1^\textit{th}}
\newcommand{\gtfunction}{f(\x)}
\newcommand{\betat}{{\betap^{(t)}}}
\newcommand{\Pbetat}{P_{\betat}}
\newcommand{\Pbetatminus}{P_{\betap^{(t - 1)}}}

\newcommand{\pbetaz}{p_{\betap^{(0)}}}

\newcommand{\pbeta}{p_{\betap}}
\newcommand{\mbd}{\text{MBD}}
\newcommand{\foralli}{\forall i \in \{1, \ldots, m\}}
\newcommand{\oracle}{p_\betap(y \mid \x)}

\usepackage{xr}
\makeatletter
\newcommand*{\addFileDependency}[1]{
  \typeout{(#1)}
  \@addtofilelist{#1}
  \IfFileExists{#1}{}{\typeout{No file #1.}}
}
\makeatother

\title{Autofocused oracles for model-based design}

\author{
  Clara Fannjiang and Jennifer Listgarten \\
  Department of Electrical Engineering \& Computer Sciences\\
  University of California, Berkeley\\
  Berkeley, CA 94720 \\
  \texttt{\{clarafy,jennl\}@berkeley.edu}
}

\begin{document}

\maketitle

\begin{abstract}
Data-driven design is making headway into a number of application areas, including protein, small-molecule, and materials engineering. The design goal is to construct an object with desired properties, such as a protein that binds to a therapeutic target, or a superconducting material with a higher critical temperature than previously observed. To that end, costly experimental measurements are being replaced with calls to high-capacity regression models trained on labeled data, which can be leveraged in an {\it in silico} search for design candidates. However, the design goal necessitates moving into regions of the design space beyond where such models were trained. Therefore, one can ask: should the regression model be altered as the design algorithm explores the design space, in the absence of new data? Herein, we answer this question in the affirmative. In particular, we (i) formalize the data-driven design problem as a non-zero-sum game, (ii) develop a principled strategy for retraining the regression model as the design algorithm proceeds---what we refer to as \ti{autofocusing}, and (iii) demonstrate the promise of autofocusing empirically.
\end{abstract}

\section{Oracle-based design}

The design of objects with desired properties, such as novel proteins, molecules, or materials, has a rich history in bioengineering, chemistry, and materials science. In these domains, design has historically been performed through iterative, labor-intensive experimentation \cite{Arnold2018-uy} (\eg,  measuring protein binding affinity) or compute-intensive physics simulations \cite{Hafner2008-rb} (\eg, computing low-energy structures for nanomaterials). Increasingly, however, attempts are being made  to replace these costly and time-consuming steps with cheap and fast calls to a proxy regression model, trained on labeled data \cite{Yang2019-ac, Wu2019-ei, Bedbrook2019-bl, Mansouri_Tehrani2018-ux, Biswas2020-um}.
Herein, we refer to such a proxy model as an {\it oracle}, and  assume that acquisition of training data for the oracle is complete, as in \cite{Bedbrook2019-bl, Mansouri_Tehrani2018-ux, Biswas2020-um, Popova2018-ad, Gomez-Bombarelli2018}.\footnote{For many applications in protein, molecule, and material design, even if one performs iterative rounds of data acquisition, at some point, the acquisition phase concludes due to finite resources.} The key issue addressed by our work is how best to train an oracle for use in design, given fixed training data.

In contrast to the traditional use of predictive models, oracle-based design is distinguished by the fact that it seeks solutions---and therefore, will query the oracle---in regions of the design space that are not well-represented by the oracle training data. If this is not the case, the design problem is easy in that the solution is within the region of the training data. Furthermore, one does not know beforehand which parts of the design space a design procedure will navigate through. As such, a major challenge arises when an oracle is employed for design: its outputs, including its uncertainty estimates, become unreliable beyond the training data \cite{Nguyen2014-xi, Brookes2019-vw}. Successful oracle-based design thus involves an inherent trade-off between the need to stay ``near'' the training data in order to trust the oracle, and the need to depart from it in order to make improvements. While trust region approaches have been developed to help address this trade-off~\cite{Brookes2019-vw,Biswas2020-um}, herein, we take a different approach and ask: {\it what is the most effective way to use a fixed, labeled dataset to train an oracle for design}? 

\vspace{-0.2cm}
\paragraph{Contributions} We develop a novel approach to oracle-based design that specifies how to update the oracle as the design space is explored---what we call {\it autofocusing} the oracle. In particular, we (i) formalize oracle-based design as a non-zero-sum game, (ii) derive an oracle-updating strategy for seeking a Nash equilibrium, and (iii) demonstrate empirically that autofocusing holds promise for improving oracle-based design. 

\vspace{-0.2cm}

\section{Model-based optimization for design}
\label{sec:mbo} 

Design problems can be cast as seeking points in the design space, $\x \in \mathcal{X}$, that with high probability satisfy desired conditions on a property random variable, $y \in \reals$. For example, one might want to design a superconducting material by specifying its chemical composition, $\x$, such that the resulting material has critical temperature greater than some threshold, $y \geq y_\tau$, or has maximal critical temperature, $y=y_\text{max}$. We specify the desired properties using a constraint set, $S$, such as \mbox{$S = \{y \colon y \geq y_\tau\}$} for some $y_\tau$. The design goal is then to solve $\argmax_\x P(y \in S \mid \x)$. This optimization problem over the inputs, $\x$, can be converted to one over {\it distributions} over the design space \cite{Zlochin2004-eq, Brookes2020}.  Specifically, model-based optimization (MBO) seeks the parameters, $\theta$, of a ``search model'', $\searchdis$, that maximizes an objective that bounds the original objective:
\vspace{-0.2cm}
\begin{align}\label{eq:mbo}
\max_\x P(y \in S \mid \x) \geq
\max_{\theta \in \Theta} \E_{\searchdis}[P(y \in S \mid \x)] = \max_{\theta \in \Theta} \E_{\searchdis}\left[\int_{S}p(y \mid \x) dy\right].
\end{align}
The original optimization problem over $\x$, and the MBO problem over $\theta$, are equivalent when the search model has the capacity to place point masses on optima of the original objective. Reasons for using the MBO formulation include that it requires no gradients of $p(y \mid \x)$, thereby allowing the use of arbitrary oracles for design, including those that are not differentiable with respect to the design space and otherwise require specialized treatment~\cite{Linder2020}. MBO also naturally allows one to obtain not just a single design candidate, but a diverse set of candidates, by sampling from the final search distribution (whose entropy can be adjusted by adding regularization to Equation \ref{eq:mbo}). Finally, MBO introduces the language of probability into the optimization, thereby allowing coherent incorporation of  probabilistic constraints such as implicit trust regions~\cite{Brookes2019-vw}. The search model can be any parameterized probability distribution that can be sampled from, and whose parameters can be estimated using weighted maximum likelihood estimation (MLE) or approximations thereof. Examples include mixtures of Gaussians, hidden Markov models, variational autoencoders~\cite{KingmaW13}, and Potts models~\cite{Marks2011}. Notably, the search model distribution can be over discrete or continuous random variables, or a combination thereof.

We use the phrase \ti{model-based design} (\mbd) to denote use of MBO to solve a design problem. Hereafter, we focus on oracle-based \mbd, which attempts to solve Equation \ref{eq:mbo} by replacing costly and time-consuming queries
of the ground truth\footnote{We refer to the \ti{ground truth} as the distribution of direct property measurements, which are inevitably stochastic due to sensor noise.}, $p(y \mid \x)$, with calls to a trained regression model (\ie, oracle), $p_\betap(y \mid \x)$, with parameters, $\betap \in B$. Given access to a fixed dataset, $\{(\xsupi, y_i)\}_{i = 1}^n$, the oracle is typically trained once using standard techniques and thereafter considered fixed~\cite{Bedbrook2019-bl, Mansouri_Tehrani2018-ux, Biswas2020-um, Popova2018-ad, Gomez-Bombarelli2018, Brookes2019-vw, Gupta2019-yd, Killoran2017-ym}. In what follows, we describe why such a strategy is sub-optimal and how to re-train the oracle in order to better achieve design goals. First, however, we briefly review a common approach for performing MBO, as we will leverage such algorithms in our approach. 

\vspace{-0.2cm}
\subsection{Solving model-based optimization problems} 
\label{section:solving_mbo}

MBO problems are often tackled with an Estimation of Distribution Algorithm (EDA)~\cite{Bengoetxea2001, Baluja1995}, a class of iterative optimization algorithms that can be seen as Monte Carlo expectation-maximization~\cite{Brookes2020}; EDAs are also connected to the cross-entropy method~\cite{Rubinstein1997,Rubinstein1999} and reward-weighted regression in reinforcement learning~\cite{Peters2007-db}. Given an oracle, $p_\betap(y \mid \x)$, and an initial search model, $p_{\theta^{(t=0)}}$, an EDA typically proceeds at iteration $t$ with two core steps:

\begin{samepage}
\begin{enumerate}
    \item ``E-step'': 
    Sample from the current search model, $\tilde{\x_i} \sim \searchdistminus$ for all $i \in \{1,\ldots,m\}$.
     Compute a weight for each sample,  $v_i \coloneqq V(P_\betap(y \in S \mid \tilde{\x}_i))$, where $V(.)$ is a method-specific, monotonic transformation.
     \nopagebreak
    \item ``M-step'': Perform weighted MLE to yield an updated search model, $\searchdist$, which tends to have more mass where $P_\betap(y \in S \mid \x)$ is high. (Some EDAs can be seen as performing {\it maximum a posteriori} inference instead, which results in smoothed parameter updates~\cite{Brookes2019-vw}.)
\end{enumerate}
\end{samepage}

Upon convergence of the EDA, design candidates can be sampled from the final search model if it is not a point mass; one may also choose to use promising samples from earlier iterations. Notably, the oracle, $p_\betap(y \mid \x)$, remains fixed in the steps above. Next, we motivate a new formalism for oracle-based \mbd~that yields a principled approach for updating the oracle at each iteration.

\section{Autofocused oracles for model-based design}

The common approach of substituting the oracle, $p_\betap(y \mid \x)$, for the ground-truth, $p(y \mid \x)$, does not address the fact that the oracle is only likely to be reliable over the distribution from which its training data were drawn~\cite{Nguyen2014-xi, Sugiyama2007-sl, Quinonero-Candela2008-de}. To address this problem, we now reformulate the \mbd~problem as a non-zero-sum game, which suggests an algorithmic strategy for iteratively updating the oracle within any MBO algorithm.

\subsection{Model-based design as a game}

When the objective in \mbox{Equation \ref{eq:mbo}} is replaced with an oracle-based version,
\begin{align}\label{eq:mbo_oracle}
\argmax_{\theta \in \Theta} \E_{\searchdis}[P_\betap(y \in S \mid \x)],
\end{align}
the solution to the oracle-based problem will, in general, be sub-optimal with respect to the original objective that uses the ground truth, $P(y \in S \mid \x)$. This sub-optimality can be extreme due to pathological behavior of the oracle when the search model, $\searchdis$, strays too far from the training distribution during the optimization~\cite{Brookes2019-vw}.

Since one cannot access the ground truth, we seek a practical alternative wherein we can leverage an oracle, but also infer when the values of the ground-truth and oracle-based objectives (in Equations \ref{eq:mbo} and \ref{eq:mbo_oracle}, respectively) are likely to be close. To do so, we introduce the notion of the {\it oracle gap}, defined as $\E_{\searchdis}[\left|P(y \in S \mid \x) - P_\betap(y \in S \mid \x)\right|]$. When this quantity is small, then by Jensen's inequality the oracle-based and ground-truth objectives are close. Consequently, our insight for improving oracle-based design is to use the oracle that minimizes the oracle gap,
\begin{align}\label{eq:oracle_gap}
    \argmin_{\betap \in B} \textsc{OracleGap}(\theta, \beta) =
    \argmin_{\betap \in B}\E_{\searchdis}[\left|P(y \in S \mid \x) - P_\betap(y \in S \mid \x)\right|].
\end{align}
Together, Equations \ref{eq:mbo_oracle} and \ref{eq:oracle_gap} define the coupled objectives of two players, namely the search model (with parameters $\theta$) and the oracle (with parameters $\beta$), in a non-zero-sum game. To attain good objective values for both players, our goal will be to search for a Nash equilibrium---that is, a pair of values $(\theta^*, \beta^*)$ such that neither can improve its objective given the other. To do so, we develop an alternating ascent-descent algorithm, which alternates between (i) fixing the oracle parameters and updating the search model parameters to increase the objective in Equation \ref{eq:mbo_oracle} (the ascent step), and (ii) fixing the search model parameters and updating the oracle parameters to decrease the objective in Equation \ref{eq:oracle_gap} (the descent step). In the next section, we describe this algorithm in more detail. 

\paragraph{Practical interpretation of the \mbd~game.}
Interpreting the usefulness of this game formulation requires some subtlety. The claim is not that every Nash equilibrium yields a search model that provides a high value of the (unknowable) ground-truth objective in Equation \ref{eq:mbo}. However, for any pair of values, $(\theta, \beta)$, the value of the oracle gap provides a certificate on the value of the ground-truth objective. In particular, if one has an oracle and search model that yield an oracle gap of $\epsilon$, then by Jensen's inequality the ground-truth objective is within $\epsilon$ of the oracle-based objective. Therefore, to the extent that we are able to minimize the oracle gap (Equation \ref{eq:oracle_gap}), we can trust the value of our oracle-based objective (Equation \ref{eq:mbo_oracle}). Note that a small, or even zero oracle gap only implies that the oracle-based objective is trustworthy; successful design also entails achieving a \ti{high} oracle-based objective, the potential for which depends on an appropriate oracle class and suitably informative training data (as it always does for oracle-based design, regardless of whether our framework is used).

Although the oracle gap as a certificate is useful conceptually for motivating our approach, at present it is not clear how to estimate it. In our experiments, we found that we could demonstrate the benefits of autofocusing without directly estimating the oracle gap, relying solely on the principle of minimizing it.
We also note that in practice, what matters is not whether we converge to a Nash equilibrium, just as what matters in empirical risk minimization is not whether one exactly recovers the global optimum, only a useful point. That is, if we can find parameters, $(\theta, \beta)$, that yield better designs than alternative methods, then we have developed a useful method.

\subsection{An alternating ascent-descent algorithm for the \mbd~game}
\label{sect:alg}

Our approach alternates between an ascent step that updates the search model, and a descent step that updates the oracle. The ascent step is relatively straightforward as it leverages existing MBO algorithms. The descent step, however, requires some creativity. In particular, for the ascent step, we run a single iteration of an MBO algorithm as described in \cref{section:solving_mbo}, to obtain a search model that increases the objective in Equation \ref{eq:mbo_oracle}.
For the descent step, we aim to minimize the oracle gap in Equation \ref{eq:oracle_gap} by making use of the following observation (proof in \suppinfo~\cref{sect:proofs}).

\begin{samepage}
\begin{proposition}\label{prop:1} 
For any search model, $\searchdis$, if the oracle parameters, $\betap$, satisfy
\begin{align}\label{eq:oracle_kl_cond}
    \E_{\searchdis}[\kld{p(y \mid \x)}{p_\betap(y \mid \x)}] = \int_{\mathcal{X}} \kld{p(y \mid \x)}{p_\betap(y \mid \x)} \, \searchdis d\x \leq \epsilon,
\end{align}
\textit{where $\kld{p}{q}$ is the Kullback-Leibler (KL) divergence between distributions $p$ and $q$, then the following bound holds: }
\vspace{-0.1cm}
\begin{align}\label{eq:lb}
    \E_{\searchdissub}[|P(y \in S \mid \x) - P_\betap(y \in S \mid \x)|] \leq \sqrt{\frac{\epsilon}{2}}.
\end{align}
\end{proposition}
\end{samepage}
\ifthenelse{\boolean{ARXIV}}{}{\vspace{-0.2cm}}

As a consequence of Proposition \ref{prop:1}, given any search model, $\searchdis$, an oracle that minimizes the expected KL divergence in Equation \ref{eq:oracle_kl_cond} also minimizes an upper bound on the oracle gap. Our descent strategy is therefore to minimize this expected divergence.
In particular, as shown in the \suppinfo~\cref{sect:proofs}, the resulting oracle parameter update at iteration $t$ can be written as
\mbox{$\betap^{(t)} = \argmax_{\betap \in B} \E_{\searchdistsub} \E_{p(y \mid \x)} [ \log p_\betap(y \mid \x)]$}, where we refer to the objective as the log-likelihood under the search model. Although we cannot generally access the ground truth, $p(y \mid \x)$, we do have labeled training data, $\trainset$, whose labels come from the ground-truth distribution, $y_i \sim p(y \mid {\x=\xsupi})$. We therefore use importance sampling with the training distribution, $p_0(\x)$, as the proposal distribution, to obtain a now-practical oracle parameter update,
\ifthenelse{\boolean{ARXIV}}{}{\vspace{-0.1cm}}
\begin{align}\label{eq:derive_iwerm2}
    \;\;\;\;\betap^{(t)} & = \argmax_{\betap \in B}\frac{1}{n} \sum_{i = 1}^n \frac{\searchdistxsupi}{p_0(\xsupi)} \log p_\betap(\ysupi \mid \xsupi).
\end{align}
The training points, $\xsupi$, are used to estimate some model for $p_0(\x)$, while $\searchdist$ is given by the search model. We discuss the variance of the importance weights, \mbox{$w_i \coloneqq p_\theta(\xsupi) / p_0(\xsupi)$}, shortly.

Together, the ascent and descent steps amount to appending a ``Step 3'' to each iteration of the generic two-step MBO algorithm outlined in \cref{section:solving_mbo}, in which the oracle is retrained on re-weighted training data according to Equation \ref{eq:derive_iwerm2}. We call this strategy \ti{autofocusing} the oracle, as it retrains the oracle in lockstep with the search model, to keep the oracle likelihood maximized on the most promising regions of the design space. 
Pseudo-code for autofocusing can be found in the \suppinfo~(Algorithms \ref{alg:1} and \ref{alg:cbas}).
As shown in the experiments, autofocusing tends to improve the outcomes of design procedures, and when it does not, no harm is incurred relative to the naive approach with a fixed oracle. Before discussing such experiments, we first make some remarks. 

\subsection{Remarks on autofocusing}
\label{sect:remarks}

\paragraph{Controlling variance of the importance weights.} It is well known that importance weights can have high, even infinite, variance \cite{mcbook}, which may prevent the importance-sampled estimate of the log-likelihood from being useful for retraining the oracle effectively. That is, solving Equation \ref{eq:derive_iwerm2} may not reliably yield oracle parameter estimates that minimize the log-likelihood under the search model. To monitor the reliability of the importance-sampled estimate, one can compute and track an \ti{effective sample size} of the re-weighted training data, \mbox{$n_e \coloneqq (\sum_{i=1}^n w_i)^2 /\sum_{i=1}^n w_i^2$}, which reflects the variance of the importance weights \cite{mcbook}. If one has some sense of a suitable sample size for the application at hand (\eg, based on the oracle model capacity), then one could monitor $n_e$ and choose not to retrain when it is too small. Another variance control strategy is to use a trust region to constrain the movement of the search model, such as in~\cite{Brookes2019-vw}, which automatically controls the variance (see \suppinfo~Proposition \ref{prop:cbas}). Indeed, our experiments show how autofocusing works synergistically with a trust-region approach. Finally, two other common strategies are: (i) self-normalizing the weights, which provides a biased but consistent and lower-variance estimate \cite{mcbook}, and (ii) flattening the weights \cite{Sugiyama2007-sl} to $w_i^\alpha$ according to a hyperparameter, $\alpha \in [0, 1]$. The value of $\alpha$ interpolates between the original importance weights ($\alpha = 1$), which provide an unbiased but high-variance estimate, and all weights equal to one ($\alpha = 0$), which is equivalent to naively training the oracle (\ie, no autofocusing).

\paragraph{Oracle bias-variance trade-off.}
If the oracle equals the ground truth over all parts of the design space encountered during the design procedure, then autofocusing should not improve upon using a fixed oracle. In practice, however, this is unlikely to ever be the case---the oracle is almost certain to be misspecified and ultimately mislead the design procedure with incorrect inductive bias. 
It is therefore interesting to consider what autofocusing does from the perspective of the bias-variance trade-off of the oracle, with respect to the search model distribution. 
On the one hand, autofocusing retrains the oracle using an unbiased estimate of the log-likelihood over the search model. On the other hand, as the search model moves further away from the training data, the effective sample size available to train the oracle decreases; correspondingly, the variance of the oracle increases. In other words, when we use a fixed oracle (no autofocusing), we prioritize minimal variance at the expense of greater bias. With pure autofocusing, we prioritize reduction in bias at the expense of higher variance. Autofocusing with techniques to control the variance of the importance weights~\cite{Sugiyama2007-sl, Sugiyama2012-ix} enables us to make a suitable trade-off between these two extremes.

\paragraph{Autofocusing corrects design-induced covariate shift.} 
In adopting an importance-sampled estimate of the training objective, Equation \ref{eq:derive_iwerm2} is analogous to the classic covariate shift adaptation strategy known as importance-weighted empirical risk minimization~\cite{Sugiyama2007-sl, Sugiyama2012-ix}. We can therefore interpret autofocusing as dynamically correcting for covariate shift induced by a design procedure, where, at each iteration, a new ``test'' distribution is given by the updated search model. Furthermore, we are in the fortunate position of knowing the exact parametric form of the test density at each iteration, which is simply that of the search model. This view highlights that the goal of autofocusing is not necessarily to increase exploration of the design space, but to provide a more useful oracle wherever the search model does move (as dictated by the underlying method to which autofocusing is added).
\vspace{-.1cm}

\section{Related Work}

Although there is no cohesive literature on oracle-based design in the fixed-data setting, its use is gaining prominence in several application areas, including the design of proteins and nucleotide sequences \cite{Biswas2020-um, Brookes2019-vw, Gupta2019-yd, Killoran2017-ym, Kumar2019-og}, molecules \cite{Olivecrona2017-bz, Popova2018-ad, Gomez-Bombarelli2018}, and materials \cite{Mansouri_Tehrani2018-ux, Hautier2010-yz}. Within such work, the danger in extrapolating beyond the training distribution is not always acknowledged or addressed. In fact, proposed design procedures often are validated under the assumption that the oracle is always correct~\cite{Popova2018-ad, Linder2020, Gupta2019-yd, Killoran2017-ym, Olivecrona2017-bz}. Some exceptions include Conditioning by Adaptive Sampling (CbAS)~\cite{Brookes2019-vw}, which employs a probabilistic trust-region approach using a model of the training distribution, and~\cite{Biswas2020-um}, which uses a hard distance-based threshold. Similar in spirit to \cite{Brookes2019-vw}, Linder et al. regularize the designed sequences based on their likelihood under a model of the training distribution \cite{Linder2020-wi}. In another approach, a variational autoencoder implicitly enforces a trust region by constraining design candidates to the probabilistic image of the decoder~\cite{Gomez-Bombarelli2018}. Finally, Kumar \& Levine tackle design by learning the inverse of a ground-truth function, which they constrain to agree with an oracle, so as to discourage too much departure from the training data \cite{Kumar2019-og}. None of these approaches update the oracle in any way. However, autofocusing is entirely complementary to and does not preclude the additional use of any of these approaches. For example, we demonstrate in our experiments that autofocusing improves the outcomes of CbAS, which implicitly inhibits the movement of the search model away from the training distribution.

Related to the design problem is that of active learning in order to optimize a function, using for example Bayesian optimization~\cite{Snoek2012-wg}. Such approaches are fundamentally distinct from our setting in that they dynamically acquire new labeled data, thereby more readily allowing for correction of oracle modeling errors. In a similar spirit, evolutionary algorithms sometimes use a ``surrogate'' model of the function of interest (equivalent to our oracle), to help guide the acquisition of new data~\cite{Jin2011-ot}. In such settings, the surrogate may be updated using an {\it ad hoc} subset of the data~\cite{Le2013-wt} or perturbation of the surrogate parameters~\cite{Schmidt2008-pw}. Similarly, a recent reinforcement-learning based approach to biological sequence design relies on new data to refine the oracle when moving into a region of design space where the oracle is unreliable \cite{Angermueller2019-iclr}.

Offline reinforcement learning (RL)~\cite{Levine2020-hd} shares similar characteristics with our problem in that the goal is to find a policy that optimizes a reward function, given only a fixed dataset of trajectories sampled using another policy. In particular, offline model-based RL leverages a learned model of dynamics that may not be accurate everywhere. Methods in that setting have attempted to account for the shift away from the training distribution using uncertainty estimation and  trust-region approaches~\cite{Chua2018-ko, Deisenroth2011-xt, Rhinehart2018-cm}; importance sampling has also been used for off-policy evaluation \cite{Precup2001-qm, Thomas2016-bb}.

As noted in the previous section, autofocusing operates through iterative retraining of the oracle in order to correct for covariate shift induced by the movement of the search model. It can therefore be connected to ideas from domain adaptation more broadly \cite{Quinonero-Candela2008-de}. Finally, we note that mathematically, oracle-based \mbd~is related to the decision-theoretic framework of performative prediction \cite{Perdomo2020-fa}. Perdomo \etal~formalize the phenomenon in which using predictive models to perform actions induces distributional shift, then present theoretical analysis of repeated retraining with new data as a solution. Our problem has two major distinctions from this setting: first, the ultimate goal in design is to maximize an unknowable ground-truth objective, not to minimize risk of the oracle. The latter is only relevant to the extent that it helps us achieve the former, and our work operationalizes that connection by formulating and minimizing the oracle gap. Second, we are in a fixed-data setting. Our work demonstrates the utility of adaptive retraining even in the absence of new data.

\section{Experiments}

We now demonstrate empirically, across a variety of both experimental settings and MBO algorithms, how autofocusing can help us better achieve design goals. First we leverage an intuitive example to gain detailed insights into how autofocus behaves. We then conduct a detailed study on a more realistic problem of designing superconducting materials. Code for our experiments is available at \url{https://github.com/clarafy/autofocused_oracles}.

\subsection{An illustrative example}
\label{sect:toy}

\begin{figure}
  \centering
  \includegraphics[width=\linewidth]{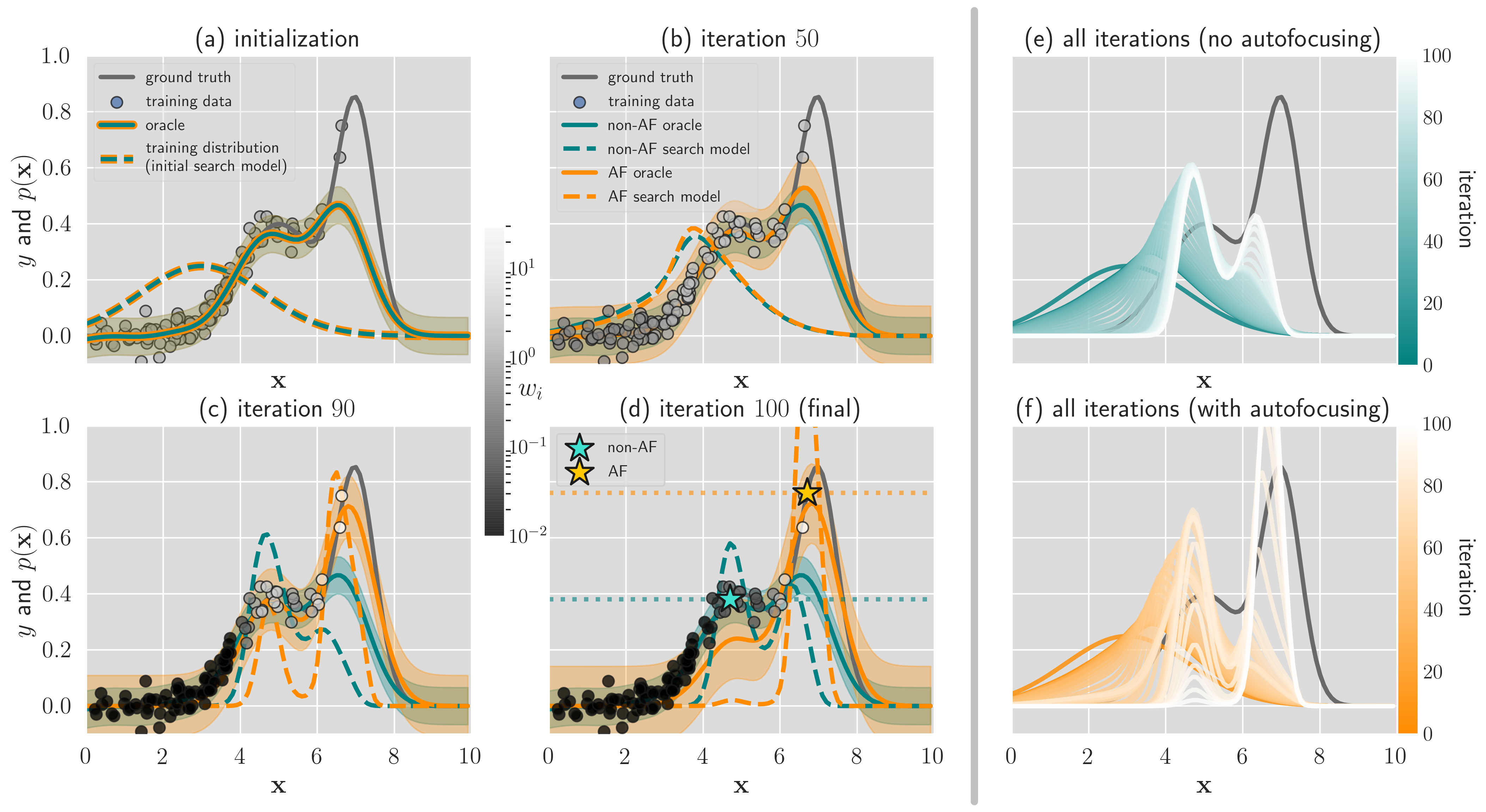}
  \caption{Illustrative example. Panels (a-d) show detailed snapshots of the MBO algorithm, CbAS~\cite{Brookes2019-vw}, with and without autofocusing (AF) in each panel. The vertical axis represents both $y$ values (for the oracle and ground truth) and probability density values (of the training distribution, $\traindis$, and search distributions, $\searchdist$). Shaded envelopes correspond to $\pm1$ standard deviation of the oracles, $\sigma_{\beta^{(t)}}$, with the oracle expectations, $\mu_{\beta^{(t)}}(\x)$, shown as a solid line. Specifically, (a) at initialization, the oracle and search model are the same for AF and non-AF. Intermediate and final iterations are shown in (b-d), where the non-AF and AF oracles and search models increasingly diverge. Greyscale of training points corresponds to their importance weights used for autofocusing.
  In (d), each star and dotted horizontal line indicate the ground-truth value corresponding to the point of maximum density, indicative of the quality of the final search model (higher is better). The values of $(\sigma_\epsilon, \sigma_0)$ used here correspond to the ones marked by an $\times$ in \fig \ref{fig:toy_improvement}, which summarizes results across a range of settings. Panels (e,f) show the search model over all iterations without and with autofocusing, respectively.
    }
  \label{fig:toy_demo}
  \ifthenelse{\boolean{ARXIV}}{}{\vspace{-0.4cm}}
\end{figure}

\begin{figure}[ht]
  \centering
  \includegraphics[width=\linewidth]{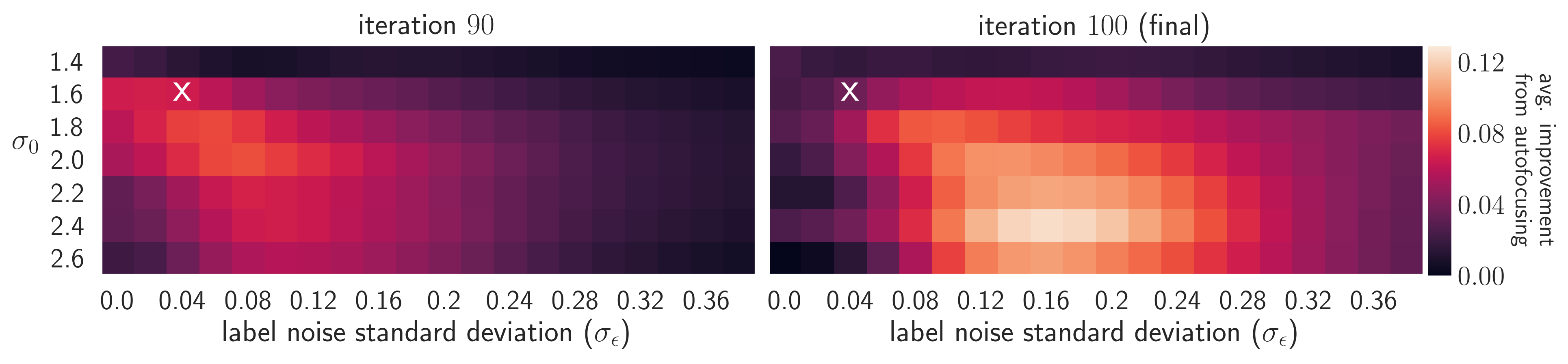}
  \caption{Improvement from autofocusing (AF) over a wide range of settings of the illustrative example. Each colored square shows the improvement (averaged over $50$ trials) conferred by AF for one setting, $(\sigma_\epsilon,\sigma_0)$, of, respectively, the standard deviations of the training distribution and the label noise. Improvement is quantified as the difference between the ground-truth objective in Equation \ref{eq:mbo} achieved by the final search model with and without AF. A positive value means AF yielded higher ground-truth values (\ie, performed better than without AF), while zero means it neither helped nor hurt. Similar plots to \fig \ref{fig:toy_demo} are shown in the \suppinfo~for other settings (\fig \ref{fig:toy_examples}).}
  \label{fig:toy_improvement}
  \ifthenelse{\boolean{ARXIV}}{}{\vspace{-0.3cm}}
\end{figure}

To investigate how autofocusing works in a setting that can be understood intuitively, we constructed a one-dimensional design problem where the goal was to maximize a multi-modal ground-truth function, $\gtfunction: \reals \rightarrow \reals^+$, given fixed training data (\fig \ref{fig:toy_demo}a).  The training distribution from which training points were drawn, $\traindis$, was a Gaussian with variance, $\sigma^2_0$, centered at $3$, a point where $\gtfunction$ is small relative to the global maximum at $7$. This captures the common scenario where the oracle training data do not extend out to global optima of the property of interest. As we increase the variance of the training distribution, $\sigma^2_0$, the training data become more and more likely to approach the global maximum of $\gtfunction$. The training labels are drawn from $p(y \mid \x) = \mathcal{N}(f(\x), \sigma_\epsilon^2)$, where $\sigma_\epsilon^2$ is the variance of the label noise. For this example, we used CbAS~\cite{Brookes2019-vw}, an MBO algorithm that employs a probabilistic trust region. We did not control the variance of the importance weights.

An MBO algorithm prescribes a sequence of search models as the optimization proceeds, where each successive search model is fit using weighted MLE to samples from its predecessor. However, in our one-dimensional example, one can instead use numerical quadrature to directly compute each successive search model~\cite{Brookes2019-vw}. Such an approach enables us to abstract out the particular parametric form of the search model, thereby more directly exposing the effects of autofocusing. In particular, we used numerical quadrature to compute the search model density at iteration $t$ as \mbox{$p^{(t)}(\x) \propto  \Pbetat(y \in S^{(t)} \mid \x) \traindis$}, where $S^{(t)}$ belongs to a sequence of relaxed constraint sets such that $S^{(t)} \supseteq S^{(t + 1)} \supseteq S$~\cite{Brookes2019-vw}. 
We computed this sequence of search models in two ways: (i) without autofocusing, that is, with a fixed oracle trained once on equally weighted training data, and (ii) with autofocusing, that is, where the oracle was retrained at each iteration.
In both cases, the oracle was of the form $\pbeta(y \mid \x) = \mathcal{N}(\mu_\beta(\x), \sigma_\beta^2)$, where $\mu_\betap(\x)$ was fit by kernel ridge regression with a radial basis function kernel and $\sigma_\betap^2$ was set to the mean squared error between $\mu_\betap(\x)$ and the labels (see \suppinfo~\cref{sect:toy_supp} for more details). Since this was a maximization problem, the desired condition was set as $S = \{y : y \geq \max_\x \mu_\betap(\x)\}$ (where $ \mu_\betap(\x) = 0.68$ for the initial oracle).
We found that autofocusing more effectively shifts the search model toward the ground-truth global maximum as the iterations proceed (\fig \ref{fig:toy_demo}b-f), thereby providing improved design candidates.

To understand the effect of autofocusing more systematically, we repeated the experiment just described across a range of settings of the variances of the training distribution, $\sigma^2_0$, and of the label noise, $\sigma^2_\epsilon$ (\fig \ref{fig:toy_improvement}).
Intuitively, both these variances control how informative the training data are about the ground-truth global maximum: as $\sigma^2_0$ increases, the training data are more likely to include points near the global maximum, and as $\sigma^2_\epsilon$ decreases, the training labels are less noisy. 
Therefore, if the training data are either too uninformative (small $\sigma^2_0$ and/or large $\sigma^2_\epsilon$) or too informative (large $\sigma^2_0$ and/or small $\sigma^2_\epsilon$), then one would not expect autofocusing to substantially improve design. In intermediate regimes, autofocusing should be particularly useful. Such a phenomenon is seen in our experiments (\fig \ref{fig:toy_improvement}).  Importantly, this kind of intermediate regime is one in which practitioners are likely to find themselves: the motivation for design is often sparked by the existence of a few examples with property values that are exceptional compared to most known examples, yet the design goal is to push the desired property to be more exceptional still. In contrast, if the true global optimum already resides in the training data, one cannot hope to design anything better anyway. However, even in regimes where autofocusing does not help, on average it does not hurt relative to a naive approach with a fixed oracle (\fig \ref{fig:toy_improvement} and \suppinfo~\cref{sect:toy}).

\ifthenelse{\boolean{ARXIV}}{}{\vspace{-0.2cm}}
\subsection{Designing superconductors with maximal critical temperature}

Designing superconducting materials with high critical temperatures is an active research problem that impacts engineering applications from magnetic resonance imaging systems to the Large Hadron Collider. To assess autofocusing in a more realistic scenario, we used a dataset comprising $21,263$ superconducting materials paired with their critical temperatures~\cite{Hamidieh2018-yu}, the maximum temperature at which a material exhibits superconductivity. Each material is represented by a feature vector of length eighty-one, which contains real-valued properties of the material's constituent elements (\eg, their atomic radius and valence). We outline our experiments here, with details deferred to the \suppinfo~\cref{sect:supercon}.

\begin{table}
  \caption{Designing superconducting materials. We ran six different MBO methods, each with and without autofocusing. For each method, we extracted those samples with oracle expectations above the $\sth{80}$ percentile and computed their ground-truth expectations. We report the median and maximum of those ground-truth expectations (both in degrees K), their percent chance of improvement (PCI, in percent) over the maximum label in the training data, as well as the Spearman correlation ($\rho$) and root mean squared error (RMSE, in degrees K) between the oracle and ground-truth expectations. Each reported score is averaged over $10$ trials, where, in each trial, a different training set was sampled from the training distribution. ``Mean Diff.'' is the average difference between the score when using autofocusing compared to not. Bold values with one star (*) and two stars (**), respectively, mean $p$-values $< 0.05$ and $< 0.01$ from a two-sided Wilcoxon signed-rank test on the $10$ paired score differences between a method with autofocus and without ('Original'). For all scores but RMSE, a higher value means autofocusing yielded better results (as indicated by the arrow $\uparrow$); for RMSE, the opposite is true (as indicated by the arrow $\downarrow$). 
  }
  \label{tab:stats}
  \tiny{
  \begin{tabular*}{\textwidth}{l @{\extracolsep{\fill}} lllll lllll}
    \toprule
    & Median $\uparrow$ & Max $\uparrow$ & PCI $\uparrow$ & $\rho \uparrow$ & RMSE $\downarrow$ & Median $\uparrow$ & Max $\uparrow$ & PCI $\uparrow$ & $\rho \uparrow$ & RMSE $\downarrow$ \\
    & \multicolumn{5}{c}{\textbf{CbAS}} & \multicolumn{5}{c}{\textbf{DbAS}} \\
    \cmidrule(r){2-6} \cmidrule(r){7-11} \\
    
    Original & 51.5 & 103.8 & 0.11 & 0.05 & 17.2 & 57.0 & 98.4 & 0.11 & 0.01 & 29.6 \\
    Autofocused & 76.4 & 119.8 & 3.78 &  0.56 & 12.9 & 78.9 & 111.6 & 4.4 & 0.01 & 24.5  \\
    Mean Diff. & \tb{24.9}** & \tb{16.0}** & \tb{3.67}** & \tb{0.51}** & \tb{-4.4}** & \tb{21.9}** & \tb{13.2}** & \tb{4.2}** & 0.01 & \tb{-5.1}* \\

    &  \multicolumn{5}{c}{\rule{0pt}{1.5em} \textbf{RWR}} & \multicolumn{5}{c}{\rule{0pt}{1.5em} \textbf{FB}} \\
    \cmidrule(r){2-6} \cmidrule(r){7-11} \\
    Original & 43.4 & 102.0 & 0.05 & 0.92 & 7.4 & 49.2 & 100.6 & 0.14 & 0.09 & 17.5  \\
    Autofocused & 71.4 & 114.0 & 1.60 & 0.65 & 12.7 & 64.2 & 111.6 & 0.86 & 0.49 & 11.1  \\
    Mean Diff. & \tb{28.0}** & \tb{12.0}** & \tb{1.56}** & \tb{-0.27}** & \tb{5.4}** & \tb{15.0}** & \tb{11.0}** & \tb{0.73}** & \tb{0.40}** & \tb{-6.4}**  \\
    
    & \multicolumn{5}{c}{\rule{0pt}{1.5em}\textbf{CEM-PI}} & \multicolumn{5}{c}{\rule{0pt}{1.5em}\textbf{CMA-ES}} \\
    \cmidrule(r){2-6} \cmidrule(r){7-11} \\
    Original & 34.5 & 55.8 & 0.00 & -0.16 & 148.3 & 42.1 & 69.4 & 0.00 & 0.27 & 27493.2 \\
    Autofocused & 67.0 & 98.0 & 1.69 & 0.13 & 29.4 & 50.2 & 85.8 & 0.01 & 0.52 & 9499.8 \\
    Mean Diff. & \tb{32.6}** & \tb{42.3}* & \tb{1.69}* & 0.29 & \tb{-118.9}** & \tb{8.1}* & \tb{16.3}* & 0.01 & \tb{0.25}* & \tb{-17993.5}* \\
    
    \bottomrule
  \end{tabular*}
  }
\ifthenelse{\boolean{ARXIV}}{}{\vspace{-0.4cm}}
\end{table}

Unlike {\it in silico} validation of a predictive model, one cannot hold out data to validate a design algorithm because one will not have ground-truth labels for proposed design candidates. Thus, similarly to~\cite{Brookes2019-vw}, we created a ``ground-truth'' model by training gradient-boosted regression trees~\cite{Hamidieh2018-yu,Stanev2018-vo} on the whole dataset and treating the output as the ground-truth expectation, $\mathbb{E}[y \mid \x]$, which can be called at any time.
Next, we generated training data to emulate the common scenario in which design practitioners have labeled data that are not dense near ground-truth global optima. In particular, we selected the $n = 17,015$ training points from the dataset whose ground-truth expectations were in the bottom $\sth{80}$ percentile. We used MLE with these points to fit a full-rank multivariate normal, which served as the training distribution, $\traindis$, from which we drew $n$ simulated training points, $\{\xsupi\}_{i = 1}^n$. For each $\xsupi$ we drew one sample, $\ysupi \sim \mathcal{N}(\E[y \mid \xsupi], 1)$, to obtain a noisy ground-truth label.
Finally, for our oracle, we used $\{(\xsupi, \ysupi)\}_{i = 1}^n$ to train an ensemble of three neural networks that output both $\mu_\betap(\x)$ and $\sigma_\betap^2(\x)$, to provide predictions of the form $p_\betap(y \mid \x) = \mathcal{N}(\mu_\betap(\x), \sigma_\betap^2(\x))$~\cite{Lakshminarayanan2016-um}.

We ran six different MBO algorithms, each with and without autofocusing, with the goal of designing materials with maximal critical temperatures. In all cases, we used a full-rank multivariate normal for the search model, and flattened the importance weights used for autofocusing to $w_i^\alpha$ \cite{Sugiyama2007-sl} with $\alpha = 0.2$ to help control variance. The MBO algorithms were: (i) Conditioning by Adaptive Sampling (CbAS) \cite{Brookes2019-vw}; (ii) Design by Adaptive Sampling (DbAS) \cite{Brookes2019-vw}; (iii) reward-weighted regression (RWR) \cite{Peters2007-db}; (iv) the ``feedback" mechanism proposed in \cite{Gupta2019-yd} (FB); (v) the cross-entropy method used to optimize probability of improvement (CEM-PI) \cite{Brookes2019-vw, Snoek2012-wg}; and (vi) Covariance Matrix Adaptation Evolution Strategy (CMA-ES) \cite{Hansen2006-sa}. These are briefly described in the \suppinfo~\cref{sect:supercon}.

To quantify the success of each algorithm, we did the following. At each iteration, $t$, we first computed the oracle expectations, \mbox{$\E_{\betap^{(t)}}[y \mid \x]$}, for each of $n$ samples drawn from the search model, $\searchdist$. We then selected the iteration where the $\sth{80}$ percentile of these oracle expectations was greatest. For that iteration, we computed various summary statistics on the \ti{ground-truth} expectations of the best samples, as judged by the oracle from that iteration (\ie, samples with oracle expectations greater than the $\sth{80}$ percentile; Table \ref{tab:stats}). See Algorithm \ref{alg:eval} in the \suppinfo~for pseudocode of this procedure. Our evaluation procedure emulates the typical setting in which a practitioner has limited experimental resources, and can only evaluate the ground truth for the most promising candidates \cite{Wu2019-ei, Bedbrook2019-bl, Mansouri_Tehrani2018-ux, Biswas2020-um}.

Across the majority of evaluation metrics, for all MBO methods, autofocusing a method provided a statistically significant improvement upon the original method. The percent chances of improvement (PCI, the percent chances that the best samples had greater ground-truth expectations than the maximum label in the training data), expose the challenging nature of the design problem. All methods with no autofocusing had a PCI less than $0.14\%$, which although small, still reflects a marked improvement over a naive baseline of simply drawing $n$ new samples from the training distribution itself, which achieves $5.9 \times 10^{-3}\%$. Plots of design trajectories from these experiments, and results from experiments without variance control and with oracle architectures of higher and lower capacities, can be found in the \suppinfo~(Figures \ref{fig:supercon_traj} and \ref{fig:supercon_traj2}, Table \ref{tab:stats_alpha1}).

\ifthenelse{\boolean{ARXIV}}{}{\vspace{-0.25cm}}
\section{Discussion}
\ifthenelse{\boolean{ARXIV}}{}{\vspace{-0.15cm}}

We have introduced a new formulation of oracle-based design as a non-zero-sum game. From this formulation, we developed a new approach for design wherein the oracle---the predictive model that replaces costly and time-consuming laboratory experiments---is iteratively retrained so as to ``autofocus'' it on the current region of design candidates under consideration. Our autofocusing approach can be applied to any design procedure that uses model-based optimization. We recommend using autofocusing with an MBO method that uses trust regions, such as CbAS~\cite{Brookes2019-vw}, which automatically helps control the variance of the importance weights used for autofocusing. For autofocusing an MBO algorithm without a trust region, practical use of the oracle gap certificate and/or effective sample size should be further investigated. Nevertheless, even without these, we have demonstrated empirically that autofocusing can provide benefits.

Autofocusing can be seen as dynamically correcting for covariate shift as the design procedure explores design space. It can also be understood as enabling a design procedure to navigate a trade-off between the bias and variance of the oracle, with respect to the search model distribution. One extension of this idea is to also perform oracle model selection at each iteration, such that the model capacity is tailored to the level of importance weight variance.

Further extensions to consider are alternate strategies for estimating the importance weights \cite{Sugiyama2012-ix}. In particular, training discriminative classifiers to estimate these weights may be fruitful when using search models that are implicit generative models, or whose likelihood cannot otherwise be computed in closed form, such as variational autoencoders~\cite{Sugiyama2012-ix, Grover2019-ne}. We believe this may be a promising approach for applying autofocusing to biological sequence design and other discrete design problems, which often leverage such models. One can also imagine extensions of autofocusing to gradient-based design procedures~\cite{Killoran2017-ym}---for example, using techniques for test-time oracle retraining, in order to evaluate the current point most accurately \cite{Sun2019-ml}.

\ifthenelse{\boolean{ARXIV}}{}{\newpage}

\begin{ack}
Many thanks to Sebasti\'an Prillo, David Brookes, Chloe Hsu, Hunter Nisonoff, Akosua Busia, and Sergey Levine for helpful comments on the work. We are also grateful to the U.S. National Science Foundation Graduate Research Fellowship Program and the Chan Zuckerberg Investigator Program for funding.
\end{ack}

\ifthenelse{\boolean{ARXIV}}{}
{
\section{Broader Impact}

If adopted more broadly, our work could affect how novel proteins, small molecules, materials, and other entities are engineered. Because predictive models are imperfect, even with the advances presented herein, care should be taken by practitioners to verify that any proposed design candidates are indeed safe and ethical for the intended downstream applications. The machine learning approach we present facilitates obtaining promising design candidates in a cost-effective manner, but practitioners must follow up on candidates proposed by our approach with conventional laboratory methods, as appropriate to the application domain.
}

\bibliographystyle{ieeetr}
\bibliography{neurips_2020}

\newpage
\renewcommand{\thesection}{S\arabic{section}}  
\renewcommand{\thetable}{S\arabic{table}}  
\renewcommand{\thefigure}{S\arabic{figure}}
\setcounter{section}{0}
\setcounter{figure}{0}

\hrule height 4pt
\vskip 0.25in
\vskip -\parskip
{\LARGE{\centering\textbf{Supplementary Material}\par}}
\vskip 0.29in
\vskip -\parskip
\hrule height 1pt
\vskip 0.09in
  
\section{Pseudocode}
\label{sect:code} 

Algorithm \ref{alg:1} gives pseudocode for autofocusing a broad class of model-based optimization (MBO) algorithms known as estimation of distribution algorithms (EDAs), which can be seen as performing Monte-Carlo expectation-maximization \cite{Brookes2020}. EDAs proceed at each iteration with a sampling-based ``E-step'' (Steps \ref{estep1} and \ref{estep2} in Algorithm \ref{alg:1}) and a weighted maximum likelihood estimation (MLE) ``M-step'' (Step \ref{mstep}; see \cite{Brookes2020} for more details). Different EDAs are distinguished by method-specific monotonic transformations $V(\cdot)$, which determine the sample weights used in the M-step. In some EDAs, this transformation is not explicitly defined, but rather implicitly implemented by constructing and using a sequence of relaxed constraint sets, $S^{(t)}$, such that \mbox{$S^{(t)} \supseteq S^{(t + 1)} \supseteq S$}~\cite{Rubinstein1997,Rubinstein1999,Brookes2019-vw}.

Algorithm \ref{alg:cbas} gives pseudocode for autofocusing a particular EDA, Conditioning by Adaptive Sampling (CbAS) \cite{Brookes2019-vw}, which uses such a sequence of relaxed constraint sets, as well as M-step weights that induce an implicit trust region for the search model update. For simplicity, the algorithm is instantiated with the design goal of maximizing the property of interest. It can easily be generalized to the goal of achieving a specific value for the property, or handling multiple properties (see Sections S2-3 of \cite{Brookes2019-vw}).

Use of [.] in the pseudocode denotes an optional input argument with default values.

\begin{algorithm}[ht]
\DontPrintSemicolon
\SetKwInOut{Input}{Input}
\SetKwInOut{Output}{Output}
\SetKwInOut{Initialization}{Initialization}{}{}
\Input{
Training data, $\trainset$; oracle model class, $\oracle$ with parameters, $\beta$, that can be estimated with MLE; search model class, $\searchdis$, with parameters, $\theta$, that can be estimated with weighted MLE or approximations thereof; desired constraint set, 
$S$ (\eg, $S = \{y \mid y \geq y_\tau.\}$); maximum number of iterations, $T$; number of samples to generate, $m$; EDA-specific monotonic transformation, $V(\cdot)$.}
\Initialization{ Obtain $\traindis$ by fitting to $\{\xsupi\}_{i = 1}^n$ with the search model class. For the search model, set $p_{\theta^{(0)}}(\x) \leftarrow \traindis$. For the oracle, $\pbetaz(y \mid \x)$, use MLE with equally weighted training data. }
\Begin{
\For{$t = 1, \ldots, T$}{
\begin{enumerate}
\item \label{estep1} Sample from the current search model, $\tildext_i \sim \searchdistminus, \foralli$.\;
\item \label{estep2} $v_i \leftarrow V(\Pbetatminus(y \in S \mid \tildext_i)), \foralli $.\;
\item \label{mstep} Fit the updated search model, $\searchdistsub$, using weighted MLE with the samples, $\sampleset$, and their corresponding EDA weights, $\{v_i\}_{i = 1}^m$.\;
\item Compute importance weights for the training data, \mbox{$w_i \leftarrow \searchdistxsupi / \searchdiszxsupi, i = 1, \ldots, n$}.\;
\item Retrain the oracle using the re-weighted training data,
\begin{align*}
    \betap^{(t)} & \leftarrow \argmax_{\betap \in B}\frac{1}{n} \sum_{i = 1}^n w_i \log p_\betap(\ysupi \mid \xsupi).
\end{align*}
\end{enumerate}
}
}
\Output{Sequence of search models, $\{\searchdist\}_{t = 1}^T$, and sequence of samples, $\{(\tildext_i, \ldots, \tildext_m)\}_{t = 1}^T$, from all iterations. One may use these in a number of different ways. For example, one may sample design candidates from the final search model, $\searchdistfinal$, or use the most promising candidates among $\{(\tildext_i, \ldots, \tildext_m)\}_{t = 1}^T$, as judged by the appropriate oracle (\ie, corresponding to the iteration at which a candidate was generated).}
\caption{\tb{Autofocused model-based optimization algorithm}}
\label{alg:1}
\end{algorithm}

\begin{algorithm}[ht]
\DontPrintSemicolon
\SetKwInOut{Input}{Input}
\SetKwInOut{Output}{Output}
\SetKwInOut{Initialization}{Initialization}{}{}
\Input{
Training data, $\trainset$; oracle model class, $\oracle$ with parameters, $\beta$, that can be estimated with MLE; search model class, $\searchdis$, with parameters, $\theta$, that can be estimated with weighted MLE or approximations thereof; maximum number of iterations, $T$; number of samples to generate, $m$; [percentile threshold, $Q = 90$].}
\Initialization{ Obtain $\traindis$ by fitting to $\{\xsupi\}_{i = 1}^n$ with the search model class.
 For the search model, set $p_{\theta^{(0)}}(\x) \leftarrow \traindis$. For the oracle, $\pbetaz(y \mid \x)$, use MLE with equally weighted training data. Set $\gamma_0 = -\infty$.}
\Begin{
\For{$t = 1, \ldots, T$}{
\begin{enumerate}
\item Sample from the current search model, $\tildext_i \sim \searchdistminus, \foralli$.\;
\item $q_t \leftarrow \text{$Q$-th percentile of the oracle expectations of the samples, } \{\mu_\beta(\tildext_i)\}_{i = 1}^m$\;
\item $\gamma_t \leftarrow \max\{\gamma_{t - 1}, q_t\}$\;
\item $v_i \leftarrow (p_0(\tildext_i) / p_{\theta^{(t - 1)}}(\tildext_i)) \Pbetatminus(y \geq \gamma_t \mid \tildext_i), \foralli $\;
\item Fit the updated search model, $\searchdistsub$, using weighted MLE with the samples, $\sampleset$, and their corresponding EDA weights, $\{v_i\}_{i = 1}^m$.\;
\item \label{step:iw} Compute importance weights for the training data, \mbox{$w_i \leftarrow \searchdistxsupi / \searchdiszxsupi, i = 1, \ldots, n$}.\;
\item \label{step:retrain} Retrain the oracle using the re-weighted training data,
\begin{align*}
    \betap^{(t)} & \leftarrow \argmax_{\betap \in B}\frac{1}{n} \sum_{i = 1}^n w_i \log p_\betap(\ysupi \mid \xsupi).
\end{align*}
\end{enumerate}
}
}
\Output{Sequence of search models, $\{\searchdist\}_{t = 1}^T$, and sequence of samples, $\{(\tildext_i, \ldots, \tildext_m)\}_{t = 1}^T$, from all iterations. One may use these in a number of different ways (see Algorithm \ref{alg:1}).}
\caption{\tb{Autofocused Conditioning by Adaptive Sampling (CbAS)}}
\label{alg:cbas}
\end{algorithm}

\begin{algorithm}[ht]
\DontPrintSemicolon
\SetKwInOut{Input}{Input}
\SetKwInOut{Output}{Output}
\SetKwInOut{Initialization}{Initialization}{}{}
\Input{
Sequence of samples, $\{(\tildext_i, \ldots, \tildext_m)\}_{t = 1}^T$, from each iteration of an MBO algorithm; their oracle expectations, $\{(\mu_{\beta^{(t)}}(\tildext_i), \ldots, \mu_{\beta^{(t)}}(\tildext_m))\}_{t = 1}^T$; [percentile threshold, $Q = 80$].}
\Begin{
\For{$t = 1, \ldots, T$}{
Compute and store $q_t \leftarrow \text{$Q$-th percentile of the oracle expectations, } \{\mu_{\beta^{(t)}}(\tildext_i)\}_{i = 1}^m$.\;
}
$t_\text{best} \leftarrow \argmax_t q_t$ (pick the best iteration) \;
$\mathcal{I} \leftarrow \{i \in \{1, \ldots, m \}: \mu_{\beta^{(t_\text{best})}}(\tilde{\x}^{(t_\text{best})}_i) \geq q_{t_\text{best}}\}$ (pick best samples from best iteration) \;
$\mu_\text{GT,best} \leftarrow \{\E[y \mid \tilde{\x}^{t_\text{best}}_i]: i \in \mathcal{I} \}$ \;
$\mu_\text{GT} \leftarrow (\E[y \mid \tilde{\x}_1^{t_\text{best}}], \ldots, \E[y \mid \tilde{\x}_m^{t_\text{best}}])$\;
$\mu_\text{oracle} \leftarrow (\mu_{\beta^{(t_\text{best})}}(\tilde{\x}^{t_\text{best}}_1), \ldots, \mu_{\beta^{(t_\text{best})}}(\tilde{\x}^{t_\text{best}}_m))$\;
$PCI \leftarrow 100 \cdot \frac{1}{|\mathcal{I}|} \sum_{i \in \mathcal{I}} \mathbf{1}[\E[y \mid \tilde{\x}^{t_\text{best}}_i] > \text{maximum label in training data})]$
$\rho \leftarrow \textsc{Spearman}(\mu_\text{GT}, \mu_\text{oracle})$\;
$RMSE \leftarrow \textsc{RMSE}(\mu_\text{GT}, \mu_\text{oracle})$\;
}
\Output{$\median(\mu_\text{GT,best}), \max(\mu_\text{GT,best}), PCI, \rho, RMSE$}
\caption{\tb{Procedure for evaluating MBO algorithms in superconductivity experiments}. For each MBO algorithm in Tables \ref{tab:stats}, \ref{tab:stats_alpha1}, \ref{tab:stats_big}, and \ref{tab:stats_small}, the reported scores were the outputs of this procedure, averaged over $10$ trials. Recall that $\mu_{\beta^{(t)}}(\x) \coloneqq \E_{\beta^{(t)}}[y \mid \x]$ denotes the expectation of the oracle model at iteration $t$, while $\E[y \mid \x]$ denotes the ground-truth expectation.}
\label{alg:eval}
\end{algorithm}

\clearpage

\section{Proofs, derivations, and supplementary results}
\label{sect:proofs}
\begin{proof}[Proof of Proposition \ref{prop:1}.]
For any distribution $\searchdis$, if
\begin{align}
    \E_{\searchdis} \left[ \kld{p(y \mid \x)}{p_\phi(y \mid \x)} \right] \leq \epsilon,
\end{align}
then it holds that
\begin{align}
    \E_{\searchdis} \left[ |P(y \in S \mid \x) - P_\phi(y \in S \mid \x) |^2 \right] & \leq \E_{\searchdis} \left[\delta(p(y \mid \x), p_\phi(y \mid \x))^2 \right] \\
    & \leq \frac{1}{2}\E_{\searchdis} \left[\kld{p(y \mid \x)}{p_\phi(y \mid \x)} \right] \\
    & \leq \frac{\epsilon}{2}.
\end{align}
where $\delta(p, q)$ is the total variation distance between probability distributions $p$ and $q$, and the second inequality is due to Pinsker's inequality. Finally, applying Jensen's inequality yields
\begin{align}
    \E_{\searchdis} \left[ \left| P(y \in S \mid \x) - P_\phi(y \in S \mid \x)  \right| \right]& \leq \sqrt{\frac{\epsilon}{2}}.
\end{align}
\end{proof}

\subsection{Derivation of the descent step to minimize the oracle gap}

Here, we derive the descent step of the alternating ascent-descent algorithm described in \cref{sect:alg}. At iteration $t$, given the search model parameters, $\theta^{(t)}$, our goal is to update the oracle parameters according to
\begin{align}\label{eq:phi_theta}
\betap^{(t)} & = \argmin_{\betap \in B} \E_{\searchdistsub}[\kld{p(y \mid \x)}{p_\betap(y \mid \x)}].
\end{align}
Note that
\begin{align}\label{eq:derive_iwerm1}
\betap^{(t)} & = \argmin_{\betap \in B} \E_{\searchdistsub}\left[\int_\reals p(y \mid \x) \log{p(y \mid \x)} dy - \int_\reals p(y \mid \x) \log{p_\betap(y \mid \x)} dy \right] \\
& = \argmax_{\betap \in B} \E_{\searchdistsub} \left[\int_\reals p(y \mid \x)\log p_\betap(y \mid \x) dy \right] \\
& = \argmax_{\betap \in B} \E_{\searchdistsub} \E_{p(y \mid \x)} [ \log p_\betap(y \mid \x) ].
\end{align}
We cannot query the ground truth, $p(y \mid \x)$, but we do have labeled training data, $\trainset$, where $\xsupi \sim \traindis$ and $\ysupi \sim p(y \mid {\x=\xsupi})$ by definition. We therefore leverage importance sampling, using $\traindis$ as the proposal distribution, to obtain
\begin{align}\label{eq:is_oracle_obj}
\betap^{(t)} = \argmax_{\betap \in B} \E_{\traindis} \E_{p(y \mid \x)} \left[ \frac{\searchdistsub}{\traindis}\log p_\betap(y \mid \x) \right].
\end{align}
Finally, we instantiate an importance sampling estimate of the objective in Equation \ref{eq:is_oracle_obj} with our labeled training data, to get a practical oracle parameter update,
\begin{align}\label{eq:oracle_update}
\betap^{(t)} & = \argmax_{\betap \in B}\frac{1}{n} \sum_{i = 1}^n \frac{\searchdistxsupi}{p_0(\xsupi)} \log p_\betap(\ysupi \mid \xsupi).
\end{align}
This update is equivalent to fitting the oracle parameters, $\beta^{(t)}$, by performing weighted MLE with the labeled training data, $\trainset$, and corresponding weights, $\{w_i\}_{i = 1}^n$, where $w_i  \coloneqq \searchdistxsupi/p_0(\xsupi)$.

\subsection{Variance of importance weights}
\label{sect:iw_var}

The importance-sampled estimate of the log-likelihood used to retrain the oracle (Equation \ref{eq:oracle_update}) is unbiased, but may have high variance due to the variance of the importance weights. To assess the reliability of the importance-sampled estimate, alongside the effective sample size described in \cref{sect:remarks}, one can also monitor confidence intervals on some loss of interest. Let $\mathcal{L}_\beta: \mathcal{X} \times \reals \rightarrow \reals$ denote a pertinent loss function induced by the oracle parameters, $\beta$, (\eg, the squared error $\mathcal{L}_\beta(\x, y) = (\E_\beta[y \mid \x] - y)^2$). The following observation is due to Chebyshev's inequality. 

\begin{propositionsupp} \label{prop:chebyshev} Suppose that $\mathcal{L}_\beta: \mathcal{X} \times \reals \rightarrow \reals$ is a bounded loss function, such that $|\mathcal{L}_\beta(\x, y)| \leq L$ for all $\x, y$, and that $p_\theta \ll p_0$. Let $\trainset$ be labeled training data such that the $\xsupi \sim p_0(\x)$ are drawn independently and $\ysupi \sim p(y \mid \x = \xsupi)$ for each $i$. For any $\delta \in (0, 1]$ and any $n > 0$, with probability at least $1 - \delta$ it holds that
\begin{align}
    \left|\E_{\searchdis} \E_{p(y \mid \x)} [\mathcal{L}_\beta(\x, y)] - \frac{1}{n} \sum_{i = 1}^n \frac{p_\theta(\xsupi)}{p_0(\xsupi)} \mathcal{L}_\beta(\xsupi, \ysupi) \right| \leq L \sqrt{\frac{d_2(p_\theta \, || \,p_0)}{n\delta}}
\end{align}
where $d_2$ is the exponentiated R\'enyi-$2$ divergence, \ie, $d_2(p_\theta \, || \, p_0) = \E_{\traindis}[(p_\theta(\x)/p_0(\x))^2]$.
\end{propositionsupp}
\begin{proof}
We use the following lemma to bound the variance of the importance sampling estimate of the loss. Chebyshev's inequality then yields the desired result.
\end{proof}
\begin{lemmasupp}[Adaptation of Lemma 4.1 in Metelli et al. (2018) \cite{Metelli2018-eh}]
Under the same assumptions as Proposition \ref{prop:chebyshev}, the joint distribution $\searchdis p(y \mid \x)$ is absolutely continuous with respect to the joint distribution $\traindis p(y \mid \x)$. Then for any $n > 0$, it holds that
\begin{align}
    \Var_{\searchdis p(y \mid \x)}\left[\frac{1}{n} \sum_{i = 1}^n \frac{p_\theta(\xsupi)}{p_0(\xsupi)} \mathcal{L}_\beta(\xsupi, \ysupi)\right] & \leq \frac{1}{n} L^2 d_2(p_\theta || p_0).
\end{align}
\end{lemmasupp}
One can use Proposition \ref{prop:chebyshev} to construct a confidence interval on, for example, the expected squared error between the oracle and the ground-truth values with respect to $\searchdis$, using the labeled training data on hand. 
The R\'enyi divergence can be estimated using, for example, the plug-in estimate $(1/n)\sum_{i = 1}^n (p_\theta(\xsupi) / p_0(\xsupi))^2$. 
While the bound, $L$, on $\mathcal{L}_\beta$ may be restrictive in general, for any given application one may be able to use domain-specific knowledge to estimate $L$. For example, in designing superconducting materials with maximized critical temperature, one can use an oracle architecture whose outputs are non-negative and at most some plausible maximum value $M$ (in degrees Kelvin) according to superconductivity theory; one could then take $L = M^2$ for squared error loss. Computing a confidence interval at each iteration of a design procedure then allows one to monitor the error of the retrained oracle.

Monitoring such confidence intervals, or the effective sample size, is most likely to be useful for design procedures that do not have in-built mechanisms for restricting the movement of the search distribution away from the training distribution. Various algorithmic interventions are possible---one could simply terminate the procedure if the error bounds, or effective sample size, surpass some threshold, or one could decide not to retrain the oracle for that iteration. For simplicity and clarity of exposition, we did not use any such interventions in this paper, but we mention them as potential avenues for further improving autofocusing in practice. Note that 1) the bound in Proposition \ref{prop:chebyshev} is only useful if the importance weight variance is finite, and 2) estimating the bound itself requires use of the importance weights, and thus may also be susceptible to high variance. It may therefore be advantageous to use a liberal criterion for any interventions.

\paragraph{CbAS naturally controls the importance weight variance.} Design procedures that leverage a trust region can naturally bound the variance of the importance weights. For instance, CbAS \cite{Brookes2019-vw}, developed in the context of an oracle with fixed parameters, $\beta$, proposes estimating the training distribution conditioned on $S$ as the search model:
\begin{align}\label{eq:cbas}
\searchdis &  = p_0(\x \mid S) = P_\betap(S \mid \x) p_0(\x) / P_0(S),
\end{align}
where \mbox{$P_0(S) = \int P_\betap(S \mid \x) p_0(\x) dx$}. This prescribed search model yields the following importance weight variance.

\begin{propositionsupp} \label{prop:cbas} For $\searchdis = p_0(\x \mid S)$, it holds that
\begin{align}
    \Var_{\traindis}\left( \frac{\searchdis}{p_0(\x)} \right) = \frac{1}{P_0(S)} - 1.
\end{align}
\end{propositionsupp}
That is, so long as $S$ has non-neglible mass under data drawn from the training distribution, $\traindis$, we have reasonable control on the variance of the importance weights. Of course, if $P_0(S)$ is too small, this bound is not useful, but to have any hope of successful data-driven design it is only reasonable to expect this quantity to be non-negligible. This is similar to the experimental requirement, in directed evolution for protein design, that the initial data exhibit some ``minimal functionality'' with regards to the property of interest \cite{Yang2019-ac}.
\begin{proof}
The variance of the importance weights can be written as

\begin{align}
    \Var_{\traindis}\left( \frac{p_0(\x \mid S)}{\traindis} \right) = d_2(p_0(\x \mid S) || \traindis) - 1,
\end{align}
where $d_2(p_0(\x \mid S) || \traindis) = \E_{\traindis}[(p_0(\x \mid S) / \traindis)^2]$ is the exponentiated R\'enyi-$2$ divergence. Then we have
\begin{align}
    \Var_{\traindis}\left( \frac{\searchdis}{\traindis} \right) & = d_2(p_0(\x \mid S) || \traindis) - 1 = \frac{1}{p_0(S)} - 1,
\end{align}
where the second equality is due to the property in Example 1 of \cite{Van_Erven2014-zs}.
\end{proof}
This variance yields the following expression for the population version of the effective sample size:
\begin{align}
    n_e^*  \coloneqq \frac{n \E_{p_0(x)}\left [ \searchdis/p_0(\x) \right]^2 }{\E_{p_0(x)}\left [ (\searchdis/p_0(\x))^2 \right]} & = \frac{n}{\E_{p_0(x)}\left [ (\searchdis/p_0(\x))^2  \right]} = n P_0(S).
\end{align}
Furthermore, CbAS proposes an iterative procedure to estimate $\searchdis$. At iteration $t$, the algorithm seeks a variational approximation to \mbox{$p^{(t)}(\x) \propto P_\betap(S^{(t)} \mid \x) p_0(\x)$}, where $S^{(t)} \supseteq S$. Since \mbox{$P_0(S^{(t)} \mid \x) \geq P_0(S \mid \x)$}, the expressions above for the importance weight variance and effective sample size for the final search model prescribed by CbAS translate into upper and lower bounds, respectively, on the importance weight variance and effective sample size for the distributions $p^{(t)}(\x)$ prescribed at each iteration.


\section{An illustrative example}
\label{sect:toy_supp}

\subsection{Experimental details}

\paragraph{Ground truth and oracle.} For the ground-truth function $f: \reals \rightarrow \reals^+$, we used the sum of the densities of two Gaussian distributions, $\mathcal{N}_1(5, 1)$ and $\mathcal{N}_2(7, 0.25)$. For the expectation of the oracle model, $\mu_\betap(\x) \coloneqq \E_\beta[y \mid \x]$, we used kernel ridge regression with a radial basis function kernel as implemented in \texttt{scikit-learn}, with the default values for all hyperparameters. The variance of the oracle model, $\sigma_\betap^2 \coloneqq \Var_\beta[y \mid \x]$, was set to the mean squared error between $\mu_\betap(\x)$ and the training data labels, as estimated with $4$-fold importance-weighted cross-validation when autofocusing \cite{Sugiyama2007-sl}.

\paragraph{MBO algorithm.} We used CbAS as follows. At iteration \mbox{$t = 1, \ldots, 100$}, similar to \cite{Brookes2019-vw}, we used the relaxed constraint set $S^{(t)} = \{y : y \geq \gamma_t \}$ where $\gamma_t$ was the $\sth{t}$ percentile of the oracle expectation, $\mu_\betap(\x)$, when evaluated over $\x \in [0, 10]$. At the final iteration, $t = 100$, the constraint set is equivalent to the design goal of maximizing the oracle expectation, $S^{(100)} = S = \{y: y \geq \max_\x \mu_\betap(\x)\}$, which is the oracle-based proxy to maximizing the ground-truth function, $f(\x)$. At each iteration, we used numerical quadrature (\texttt{scipy.integrate.quad}) to compute the search model,
\begin{align}
    p^{(t)}(\x) = \frac{  \Pbetat(y \in S^{(t)} \mid \x) \, \traindis}{\int_\mathcal{X}\Pbetat(y \in S^{(t)} \mid \x) \, \traindis}.
\end{align}
Numerical integration enabled us to use CbAS without a parametric search model, which otherwise would have been used to find a variational approximation to this distribution \cite{Brookes2019-vw}. We also used numerical integration to compute the value of the model-based design objective (Equation \ref{eq:mbo}) achieved by the final search model, both with and without autofocusing.

\subsection{Additional plots and discussion}

For all tested settings of the variance of the training distribution, $\sigma_0^2$, and the variance of the label noise, $\sigma_\epsilon^2$, autofocusing yielded positive improvement to the model-based design objective (Equation \ref{eq:mbo}) on average over $50$ trials (Figure \ref{fig:toy_improvement}). For a more comprehensive understanding of the effects of autofocusing, here we pinpoint specific trials where autofocusing decreased the objective, compared to a naive approach with a fixed oracle. Such trials were rare, and occurred in regimes where one would not reasonably expect autofocusing to provide a benefit. In particular, as discussed in \cref{sect:toy}, such regimes include when $\sigma_0^2$ is too small, such that training data are unlikely to be close to the global maximum, and when $\sigma_0^2$ is too large, such that the training data already include points around the global maximum and a fixed oracle should be suitable for successful design. Similarly, when the label noise variance, $\sigma_\epsilon^2$, is too large, the training data are no longer informative and no procedure should hope to perform well systematically. We now walk through each of these regimes.

When $\sigma_0^2$ was small and there was no label noise, we observed a few trials where the final search model placed less mass under the global maximum with autofocusing than without. This effect was due to increased standard deviation of the autofocused oracle, induced by high variance of the importance weights (Figure \ref{fig:tr16lab0neg}). When $\sigma_0^2$ was small and $\sigma_\epsilon^2$ was extremely large, a few trials yielded lower final objectives with autofocusing by insignificant margins; in such cases, the label noise was overwhelming enough that the search model did not move much anyway, either with or without autofocusing (Figure \ref{fig:tr16lab38neg}). Similarly, when $\sigma_0^2$ was large and there was no label noise, a few trials yielded lower final objectives with autofocusing than without, by insignificant margins (Figure \ref{fig:tr22lab0neg}).

\begin{figure}
\begin{subfigure}{\textwidth}
  \centering
  \includegraphics[width=0.95\linewidth]{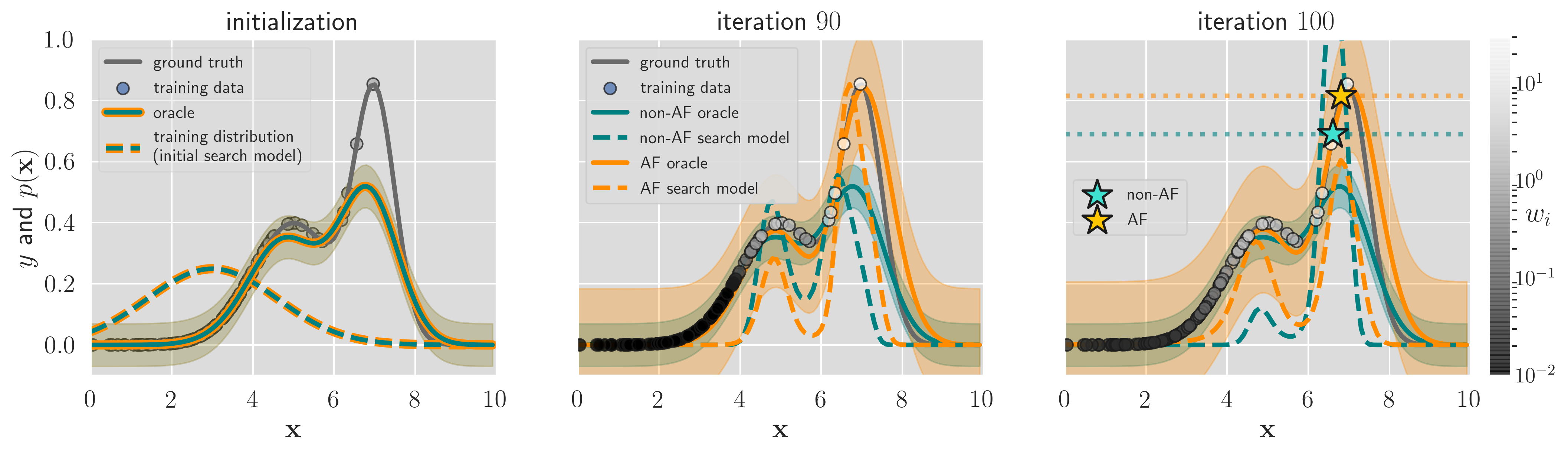}
  \caption{Example trial with low-variance training distribution and no label noise, $(\sigma_0, \sigma_\epsilon) = (1.6, 0)$.}
  \label{fig:tr16lab0neg}
\end{subfigure}
\begin{subfigure}{\textwidth}
  \centering
  \includegraphics[width=0.95\linewidth]{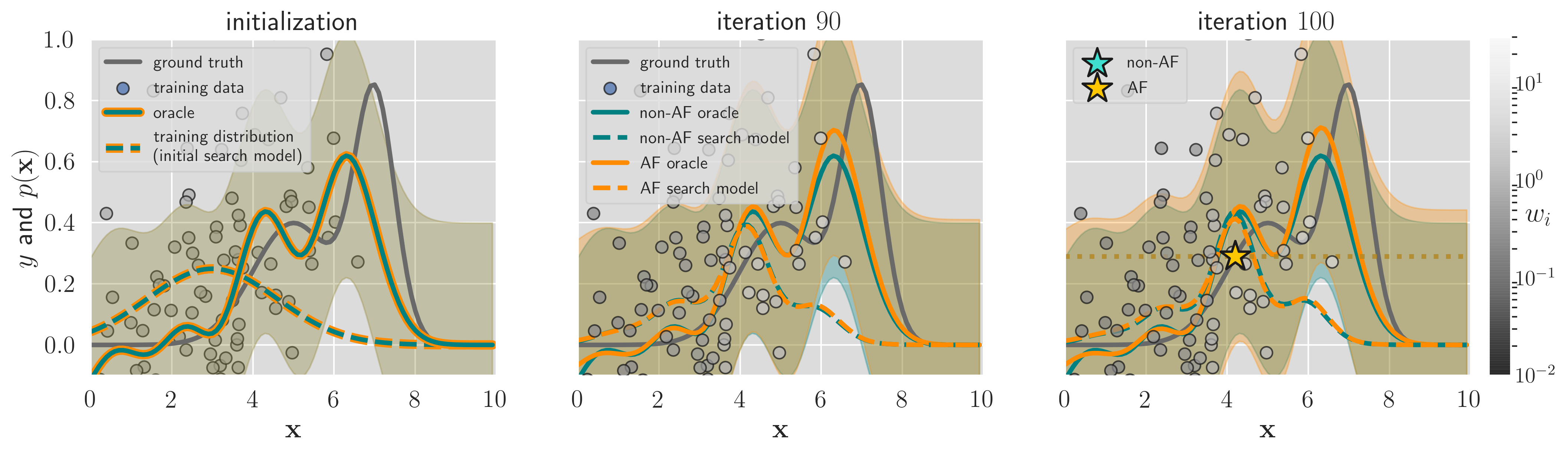}
  \caption{Example trial with low-variance training distribution and high label noise, $(\sigma_0, \sigma_\epsilon) = (1.6, 0.38)$.}
  \label{fig:tr16lab38neg}
\end{subfigure}
\begin{subfigure}{\textwidth}
  \centering
  \includegraphics[width=0.95\linewidth]{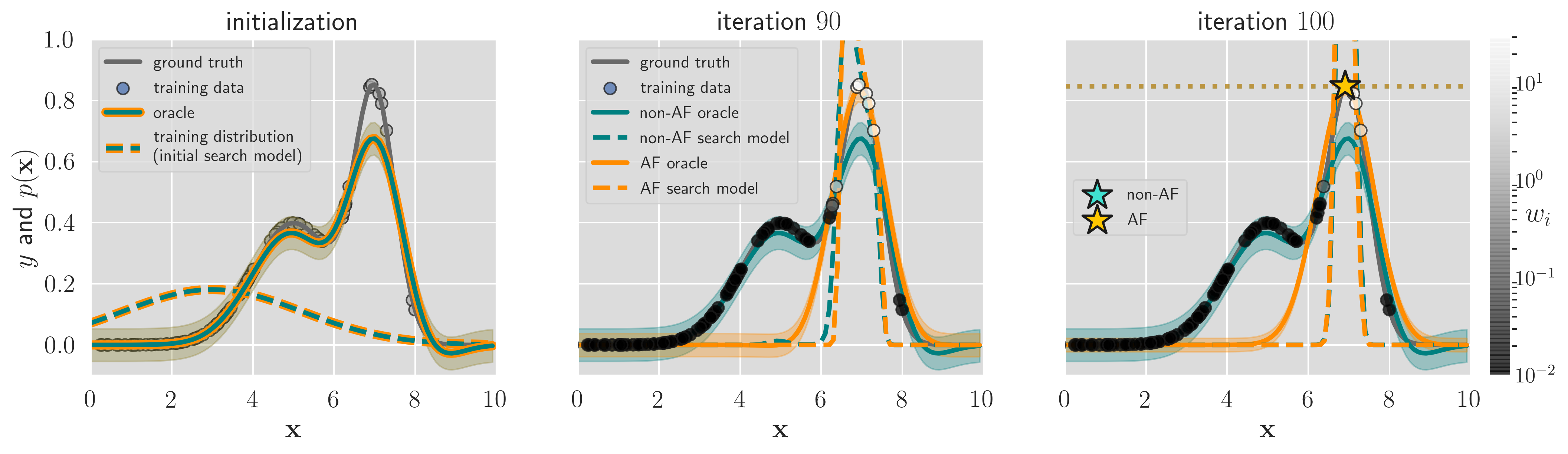}
  \caption{Example trial with high-variance training distribution and no label noise $(\sigma_0, \sigma_\epsilon) = (2.2, 0)$.}
  \label{fig:tr22lab0neg}              
\end{subfigure}
\begin{subfigure}{\textwidth}
\centering
\includegraphics[width=0.95\linewidth]{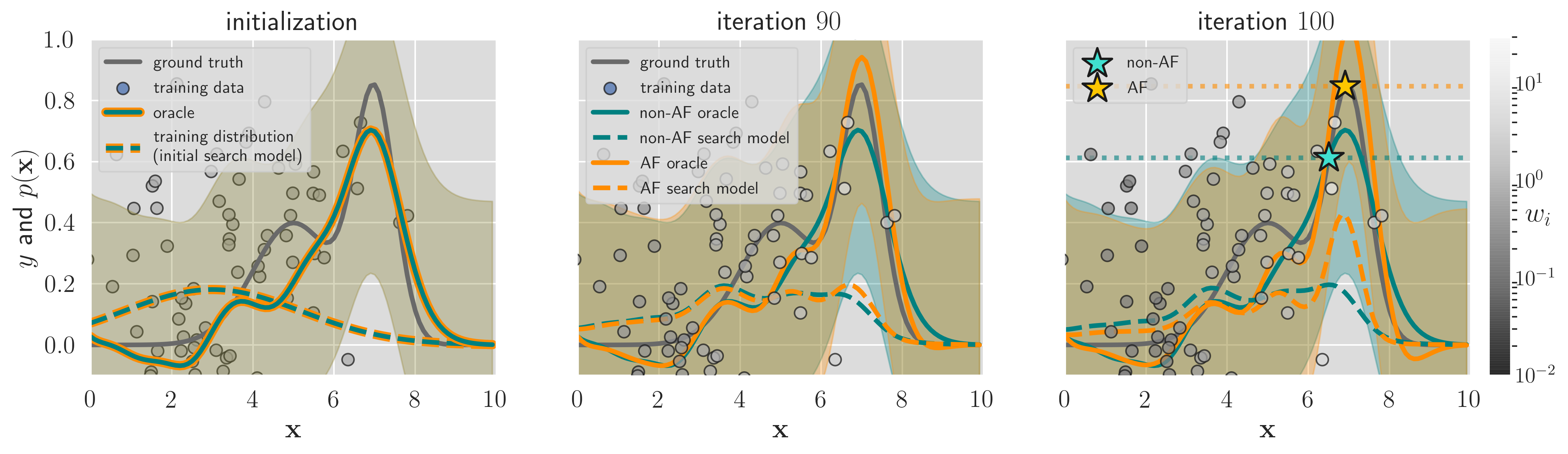}
\caption{Example trial with high-variance training distribution and high label noise $(\sigma_0, \sigma_\epsilon) = (2.2, 0.38)$.}
\label{fig:tr22lab38pos}
\end{subfigure}
\caption{Examples of regimes where autofocus (AF) sometimes yielded lower final objectives than without (non-AF). Each row shows snapshots of CbAS in a different experimental regime, from initialization (leftmost panel), to an intermediate iteration (middle panel), to the final iteration (rightmost panel).
As in Figure \ref{fig:toy_demo}, the vertical axis represents both $y$ values (for the oracle and ground truth) and probability density values (of the training distribution, $\traindis$, and search distributions, $\searchdist$). Shaded envelopes correspond to $\pm1$ standard deviation of the oracles, $\sigma_{\beta^{(t)}}$, with the oracle expectations, $\mu_{\beta^{(t)}}(\x)$, shown as a solid line. Greyscale of training points corresponds to their importance weights used in autofocusing. In the rightmost panels, for easy visualization of the final search models achieved with and without AF, the stars and dotted horizontal lines indicate the ground-truth values corresponding to the points of maximum density.}
\label{fig:toy_examples}
\end{figure}

Interestingly, when the variances of both the training distribution and label noise were high, autofocusing yielded positive improvement for all trials. In this regime, by encouraging the oracle to fit most accurately to the points with the highest labels, autofocusing resulted in search models with greater mass under the global maximum than the fixed-oracle approach, which was more influenced by the extreme label noise (Figure \ref{fig:tr22lab38pos}).

As discussed in \cref{sect:toy}, in practice it is often the case that 1) practitioners can collect reasonably informative training data for the application of interest, such that some exceptional examples are measured (corresponding to sufficiently large $\sigma_0^2$), and 2) there is always label noise, due to measurement error (corresponding to non-zero $\sigma_\epsilon^2$). Thus, we expect many design applications in practice to fall in the intermediate regime where autofocusing tends to yield positive improvements over a fixed-oracle approach (Figure \ref{fig:toy_improvement}, Table \ref{tab:stats}).

\section{Designing superconductors with maximal critical temperature}
\label{sect:supercon}

\subsection{Experimental details}

\begin{figure}
\centering
\includegraphics[width=\linewidth]{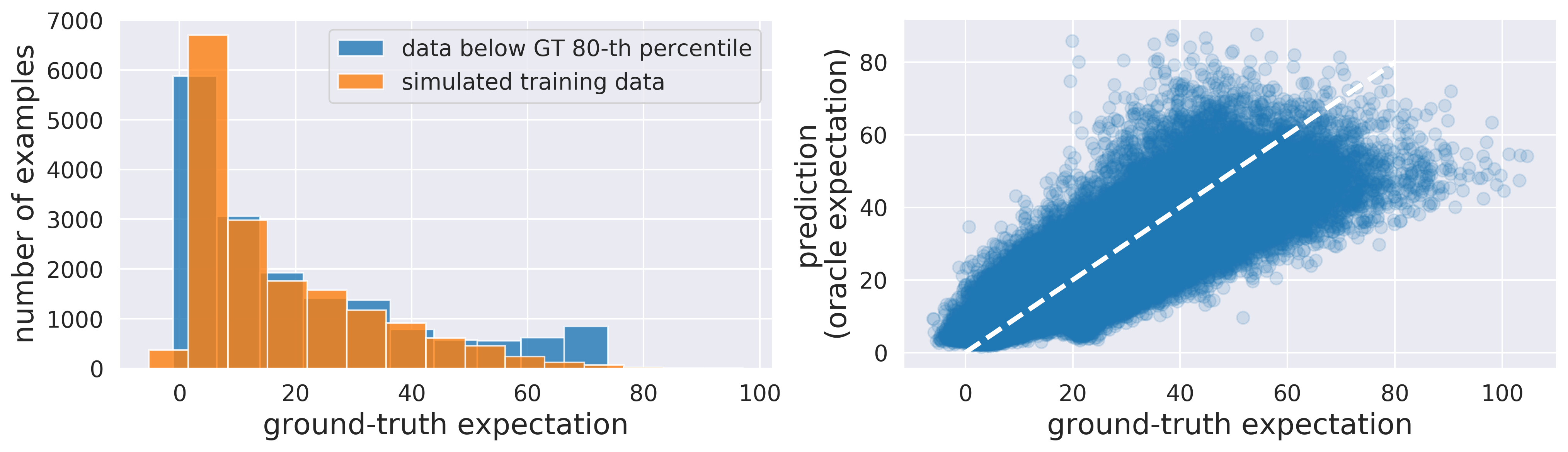}
\caption{Training distribution and initial oracle for designing superconductors. Simulated training data were generated from a training distribution, $\traindis$, which was a multivariate Gaussian fit to data points with ground-truth expectations below the $\sth{80}$ percentile. The left panel shows histograms of the ground-truth expectations of these original data points, and the ground-truth expectations of simulated training data. The right panel illustrates the error of an initial oracle used in the experiments, by plotting the ground-truth and predicted labels of $10,000$ test points drawn from the training distribution. The RMSE here was $7.31$.}
\label{fig:supercon_init}
\end{figure}

\paragraph{Pre-processing.} Each of the $21,263$ materials in the superconductivity data from \cite{Hamidieh2018-yu} is represented by a vector of eighty-one real-valued features. We zero-centered and normalized each feature to have unit variance.

\paragraph{Ground-truth model.} To construct the model of the ground-truth expectation, $\E[y \mid \x]$, we fit gradient-boosted regression trees using \texttt{xgboost} and the same hyperparameters reported in \cite{Hamidieh2018-yu}, which selected them using grid search. The one exception was that we used $200$ trees instead of $750$ trees, which yielded a hold-out root mean squared error (RMSE) of $9.51$ compared to the hold-out RMSE of $9.5$ reported in \cite{Hamidieh2018-yu}. To remove collinear features noted in \cite{Hamidieh2018-yu}, we also performed feature selection by thresholding \texttt{xgboost}'s in-built feature weights, which quantifies how many times a feature is used to split the data across all trees. We kept the sixty most important features according to this score, which decreased the hold-out RMSE from $9.51$ when using all the features to $9.45$. The resulting input space for design was then $\mathcal{X} = \reals^{60}$.

\paragraph{Training distribution.} To construct the training distribution, we selected the $17,015$ points from the dataset whose ground-truth expectations were below the $\sth{80}$ percentile (equivalent to $73.8$ degrees Kelvin, compared to the maximum of $138.3$ degrees Kelvin in the full dataset). We used MLE with these points to fit a full-rank multivariate normal, which served as the training distribution, $\traindis$, from which we drew $n = 17,015$ simulated training points, $\{\xsupi\}_{i = 1}^n$, for each trial. For each $\xsupi$ we drew one sample, $\ysupi \sim \mathcal{N}(\E[y \mid \xsupi], 1)$, to obtain a noisy ground-truth label. This training distribution produced simulated training points with a distribution of ground-truth expectations, $\E[y \mid \x]$, reasonably comparable to that of the points from the original dataset (Figure \ref{fig:supercon_init}, left panel).

\paragraph{Oracle.} For the oracle, we trained an ensemble of three neural networks to maximize log-likelihood according to the method described in \cite{Lakshminarayanan2016-um} (without adversarial examples). Each model in the ensemble had the architecture $\texttt{Input(60)} \rightarrow \texttt{Dense(100)} \rightarrow \texttt{Dense(100)} \rightarrow \texttt{Dense(100)} \rightarrow \texttt{Dense(100)} \rightarrow \texttt{Dense(10)} \rightarrow \texttt{Dense(2)}$, with $\texttt{elu}$ nonlinearities everywhere except for linear output units. Out of the range of hidden layer numbers and sizes we tested, this architecture minimized RMSE on held-out data. Each model was trained using Adam \cite{Kingma2014-ty} with a learning rate of $5 \times 10^{-4}$ for a maximum of $2000$ epochs, with a batch size of $64$ and early stopping based on the log-likelihood of a validation set. Across the $10$ trials, the initial oracles had hold-out RMSEs between $6.95$ and $7.40$ degrees Kelvin (Figure \ref{fig:supercon_init}, right panel).

\paragraph{Autofocusing.} During autofocusing, each model in the oracle ensemble was retrained with the re-weighted training data, using the same optimization hyperparameters as the initial oracle, except early stopping was based on the re-weighted log-likelihood of the validation set. For the results in Table \ref{tab:stats}, to help control the variance of the importance weights, we flattened the importance weights to $w_i^\alpha$ where $\alpha = 0.2$ \cite{Sugiyama2007-sl} and also self-normalized them \cite{mcbook}. We found that autofocusing yielded similarly widespread benefits for a wide range of values of $\alpha$, including $\alpha = 1$, which corresponds to a ``pure'' autofocusing strategy without variance control (Table \ref{tab:stats_alpha1}).

\paragraph{MBO algorithms.} Here, we provide a brief description of the different MBO algorithms used in the superconductivity experiments (Tables \ref{tab:stats}, \ref{tab:stats_alpha1}, \ref{tab:stats_big}, \ref{tab:stats_small}, Figures \ref{fig:supercon_traj} and \ref{fig:supercon_traj2}). Wherever applicable, in parentheses we anchor these descriptions in the notation and procedure of Algorithm \ref{alg:1}.
\begin{itemize}
    \item \ti{Design by Adaptive Sampling (DbAS)} \cite{Brookes2019-vw}. A basic EDA that anneals a sequence of relaxed constraint sets, $S^{(t)}$, such $S^{(t)} \supseteq S^{(t + 1)} \supseteq S$, to iteratively solve the oracle-based MBD problem (Equation \ref{eq:mbo_oracle}). (At iteration $t$, DbAS uses \mbox{$V(\tildext_i) = \Pbetatminus(y \in S^{(t)} \mid \tildext_i)$}.)

    \item \ti{Conditioning by Adaptive Sampling (CbAS)} \cite{Brookes2019-vw}. Seeks to estimate the training distribution conditioned on the desired constraint set $S$ (Equation \ref{eq:cbas}). Similar mechanistically to DbAS, as it involves constructing a sequence of relaxed constraint sets, but also incorporates an implicit trust region based on the training distribution. (At iteration $t$, CbAS uses \mbox{$V(\tildext_i) = (p_0(\tildext_i) / p_{\theta^{(t - 1)}}(\tildext_i)) \Pbetatminus(y \in S^{(t)} \mid \tildext_i)$}. See Algorithm \ref{alg:cbas}; non-autofocused CbAS excludes Steps \ref{step:iw} and \ref{step:retrain}.)
    
    \item \ti{Reward-Weighted Regression (RWR)} \cite{Peters2007-db}. An EDA used in the reinforcement learning community. (At iteration $t$, RWR uses \mbox{$V(\tildext_i) = v'_i / \sum_{j = 1}^m v'_j$}, where \mbox{$v'_i = \exp(\gamma \E_{\beta^{(t - 1)}}[y \mid \tildext_i]))$} and $\gamma > 0$ is a hyperparameter).
    
    \item \ti{``Feedback'' Mechanism (FB)} \cite{Gupta2019-yd}. A heuristic version of CbAS, which maintains samples from previous iterations to prevent the search model from changing too rapidly. (At Step \ref{mstep} in Algorithm \ref{alg:1}, FB uses samples from the current iteration with oracle expectations that surpass some percentile threshold, along with a subset of promising samples from previous iterations.)
    
    \item \ti{Cross-Entropy Method with Probability of Improvement (CEM-PI)} \cite{Brookes2019-vw}. A baseline EDA that uses the cross-entropy method \cite{Rubinstein1997, Rubinstein1999} to maximize the probability of improvement, an acquisition function commonly used in Bayesian optimization \cite{Snoek2012-wg}. (At iteration $t$, CEM-PI uses \mbox{$V(\tildext_i) = \mathbf{1}[\Pbetat(y \geq y_\text{max} \mid \tildext_i) \geq \gamma_t]$}, where $y_\text{max}$ is the maximum label observed in the training data, and, following the cross-entropy method, $\gamma_t$ is some percentile of the probabilities of improvement according to the oracle, $\{\Pbetat(y \geq y_\text{max} \mid \tildext_i)\}_{i = 1}^m$.)
    
    \item \ti{Covariance Matrix Adaptation Evolution Strategy (CMA-ES)} ~\cite{Hansen2006-sa}. A state-of-the-art EDA developed for the special case of multivariate Gaussian search models. We used it to maximize the probability of improvement according to the oracle, $\Pbetat(y \geq y_\text{max} \mid \tildext_i)$.
\end{itemize}

CbAS, DbAS, FB, and CEM-PI all have hyperparameters corresponding to a percentile threshold (for CbAS and DbAS, this is used to construct the relaxed constraint sets). We set this hyperparameter to $90$ for all these methods. For RWR, we set $\gamma = 0.01$, and for CMA-ES, we set the step size hyperparameter to $\sigma = 0.01$.

\begin{figure}
\begin{subfigure}{\textwidth}
  \centering
  \includegraphics[width=\linewidth]{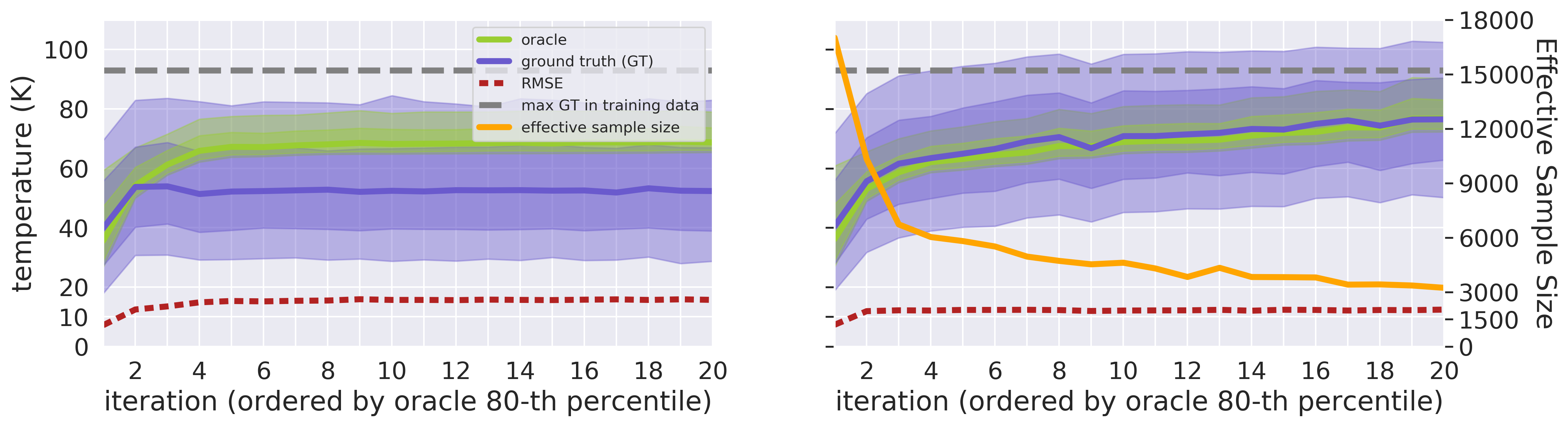}
  \caption{CbAS.}
  \label{fig:cbas}
\end{subfigure}
\begin{subfigure}{\textwidth}
  \centering
  \includegraphics[width=\linewidth]{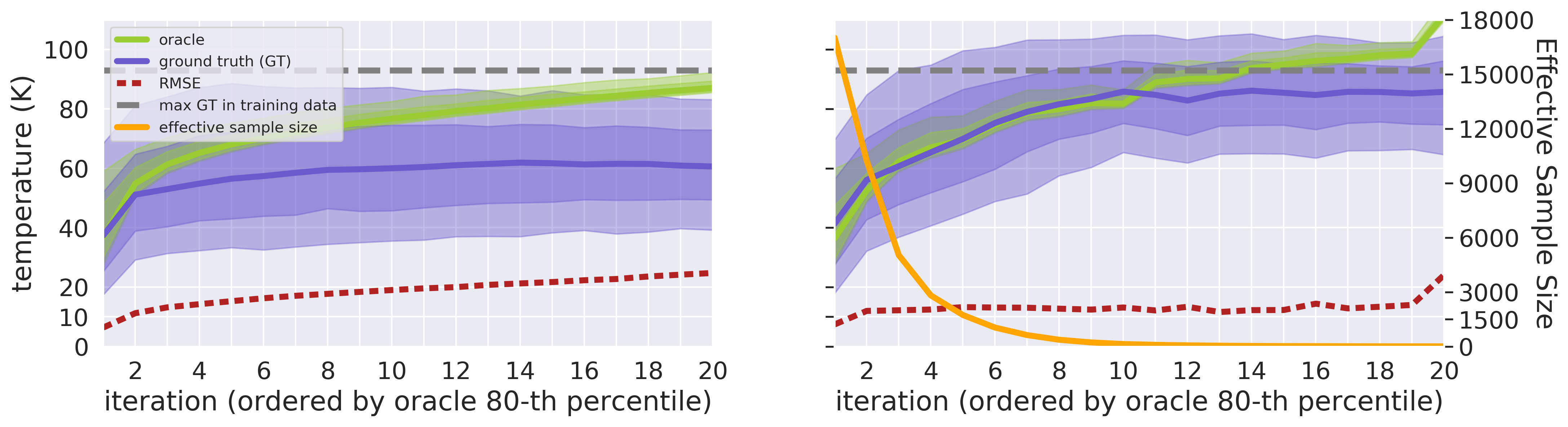}
  \caption{DbAS.}
  \label{fig:dbas}
\end{subfigure}
\begin{subfigure}{\textwidth}
  \centering
  \includegraphics[width=\linewidth]{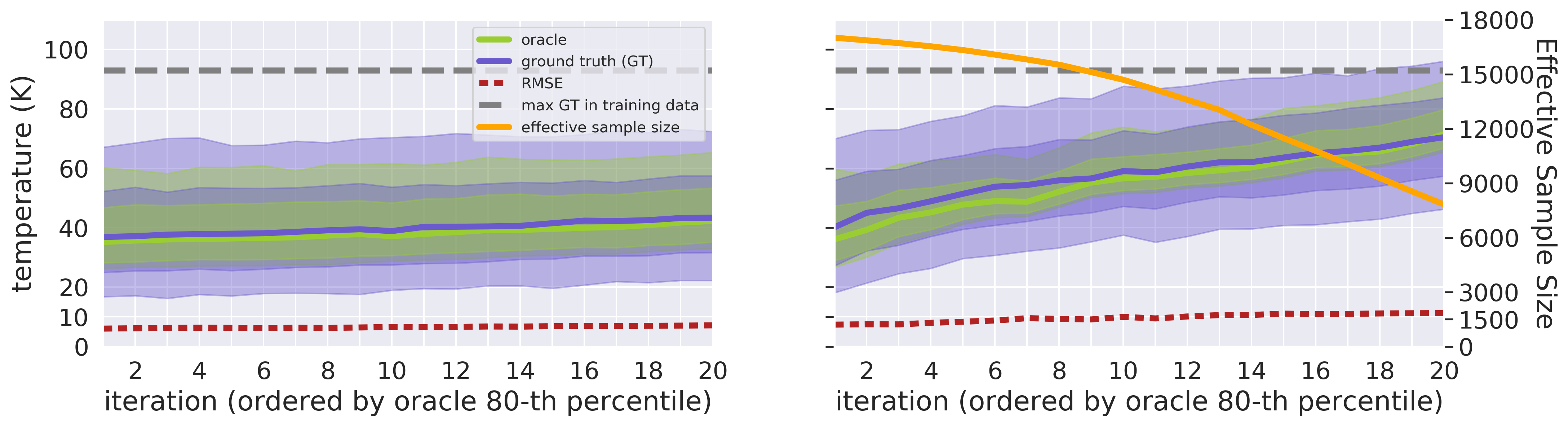}
  \caption{RWR.}
  \label{fig:rwr}
\end{subfigure}
\begin{subfigure}{\textwidth}
  \centering
  \includegraphics[width=\linewidth]{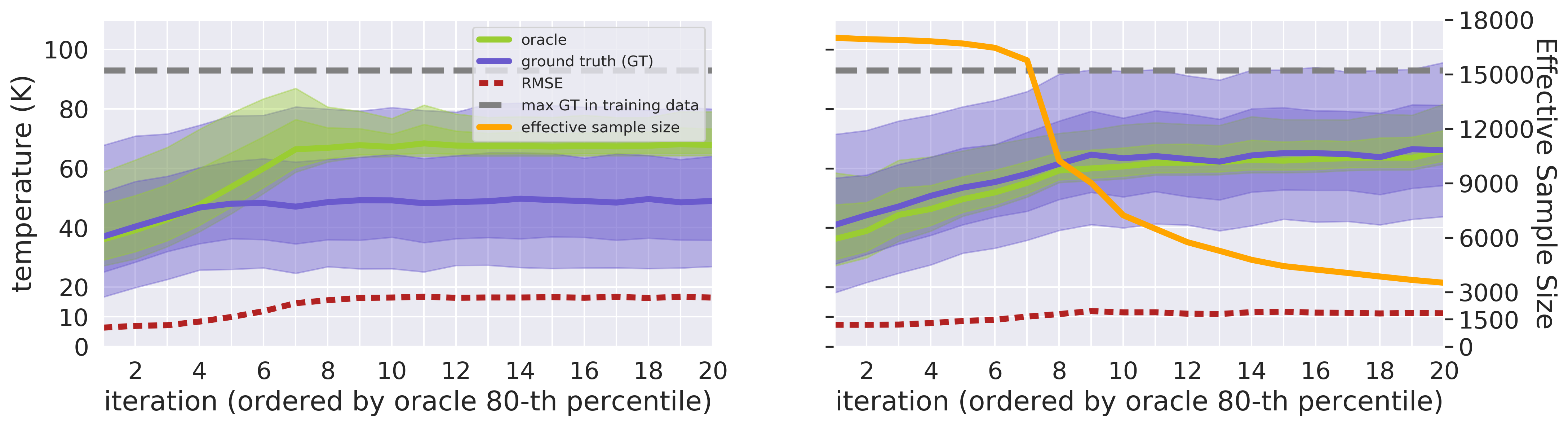}
  \caption{FB.}
  \label{fig:fb}
\end{subfigure}
\caption{Designing superconducting materials. Trajectories of different MBO algorithms run without (left) and with autofocusing (right), on one example trial used to compute Table \ref{tab:stats}. 
At each iteration, we extract the samples with oracle expectations greater than the $\sth{80}$ percentile. For these samples, solid lines give the median oracle (green) and ground-truth (indigo) expectations. The shaded regions capture $70$ and $95$ percent of these quantities. The RMSE at each iteration is between the oracle and ground-truth expectations of all samples. The horizontal axis is sorted by increasing $\sth{80}$ percentile of oracle expectations (\ie, the samples plotted at iteration $1$ are from the iteration whose $\sth{80}$ percentile of oracle expectations was lowest). This ordering exposes the trend of whether the oracle expectations of samples were correlated to their ground-truth expectations. Two more algorithms are shown in Figure \ref{fig:supercon_traj2}.
}
\label{fig:supercon_traj}
\end{figure}

\begin{figure}
\begin{subfigure}{\textwidth}
  \centering
  \includegraphics[width=\linewidth]{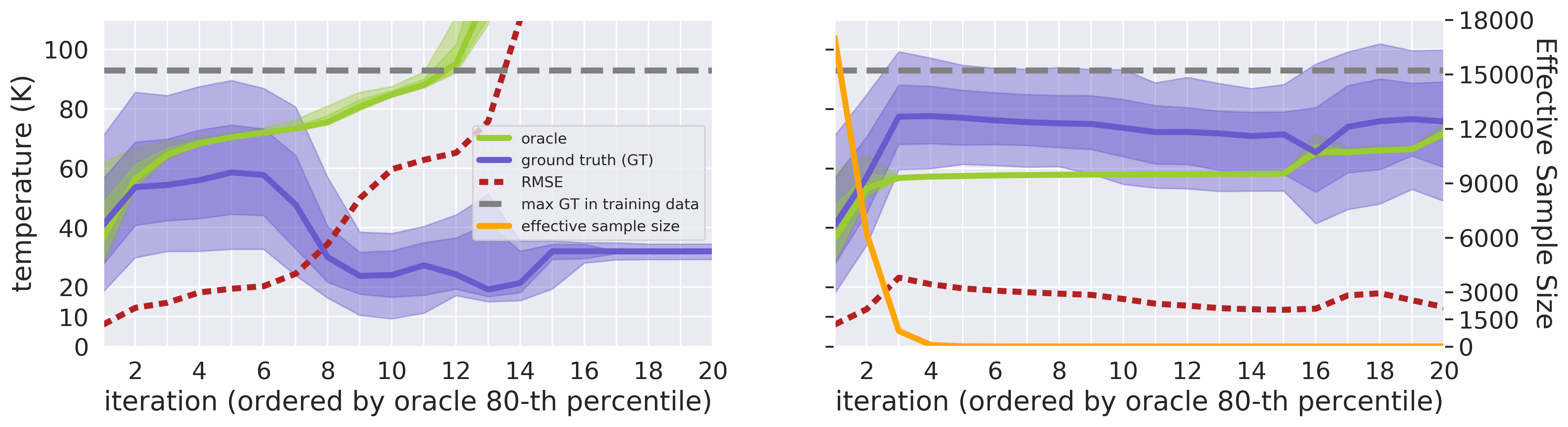}
  \caption{CEM-PI.}
  \label{fig:cempi}
\end{subfigure}
\begin{subfigure}{\textwidth}
  \centering
  \includegraphics[width=\linewidth]{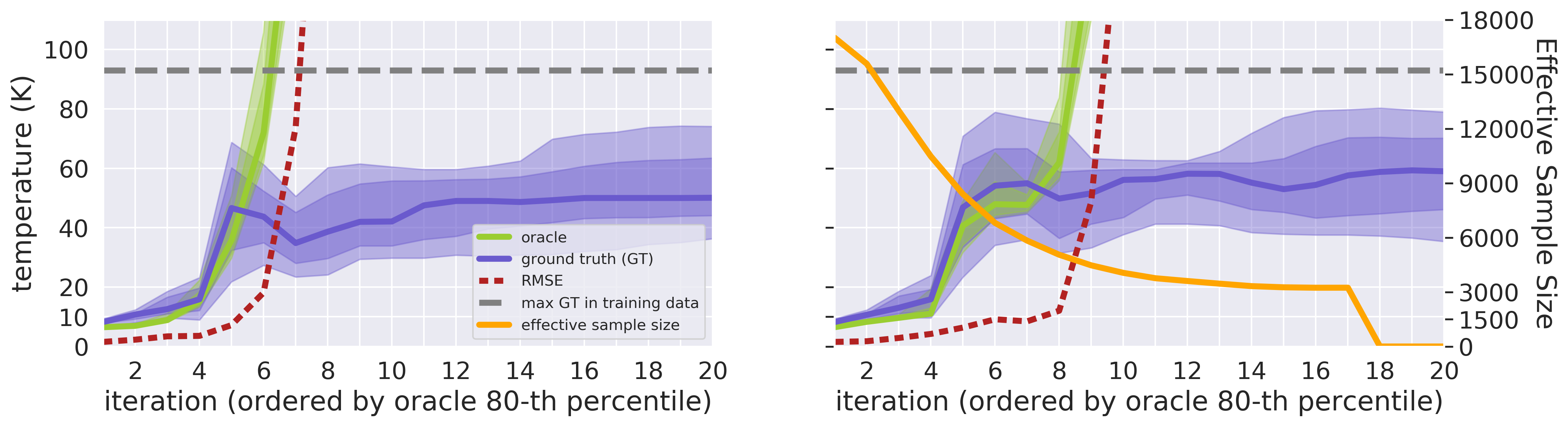}
  \caption{CMA-ES.}
  \label{fig:cmaes}
\end{subfigure}
\caption{Designing superconducting materials. Continuation of Figure \ref{fig:supercon_traj}.
}
\label{fig:supercon_traj2}
\end{figure}

\subsection{Additional experiments}

\paragraph{Importance weight variance control.} To see how much importance weight variance affects autofocusing, we conducted the same experiments as Table \ref{tab:stats}, except without flattening the weights to reduce variance (Table \ref{tab:stats_alpha1}). For CbAS, DbAS, RWR, FB, and CEM-PI, autofocusing without variance control yielded statistically significant improvements to the majority of scores, though with somewhat lesser effect sizes than in Table \ref{tab:stats} when the weights were flattened with $\alpha = 0.2$. For CMA-ES, the only significant improvement autofocusing rendered was to the Spearman correlation between the oracle and the ground-truth expectations. Note that CMA-ES is a state-of-the-art method for optimizing a given objective with a multivariate Gaussian search model \cite{Hansen2006-sa}, which likely led to liberal movement of the search model away from the training distribution and therefore high importance weight variance.

\begin{table}
  \caption{Designing superconducting materials. Same experiments and caption as Table \ref{tab:stats}, except with $\alpha = 1$ (no flattening of the importance weights to control variance). We ran six different MBO methods, each with and without autofocusing. For each method, we extracted those samples with oracle expectations above the $\sth{80}$ percentile and computed their ground-truth expectations. We report the median and maximum of those ground-truth expectations (both in degrees K), their percent chance of improvement (PCI, in percent) over the maximum label in the training data, as well as the Spearman correlation ($\rho$) and root mean squared error (RMSE, in degrees K) between the oracle and ground-truth expectations. Each reported score is averaged over $10$ trials, where, in each trial, a different training set was sampled from the training distribution. ``Mean Diff.'' is the average difference between the score when using autofocusing compared to not. Bold values with one star (*) and two stars (**), respectively, mean $p$-value $< 0.05$ and $< 0.01$ from a two-sided Wilcoxon signed-rank test on the $10$ paired score differences. For all scores but RMSE, a higher value means autofocusing yielded better results (as indicated by the arrow $\uparrow$); for RMSE, the opposite is true (as indicated by the arrow $\downarrow$). 
  }
  \label{tab:stats_alpha1}
  \tiny{
  \begin{tabular*}{\textwidth}{l @{\extracolsep{\fill}} lllll lllll}
    \toprule
    & Median $\uparrow$ & Max $\uparrow$ & PCI $\uparrow$ & $\rho \uparrow$ & RMSE $\downarrow$ & Median $\uparrow$ & Max $\uparrow$ & PCI $\uparrow$ & $\rho \uparrow$ & RMSE $\downarrow$ \\
    & \multicolumn{5}{c}{\textbf{CbAS}} & \multicolumn{5}{c}{\textbf{DbAS}} \\
    \cmidrule(r){2-6} \cmidrule(r){7-11} \\
    
    Original & 51.5 & 103.8 & 0.11 & 0.05 & 17.2 & 57.0 & 98.4 & 0.11 & 0.01  & 29.6  \\
    Autofocused & 73.2 & 116.0 & 2.29 & 0.56 & 12.8 & 69.4 & 109.9 & 0.68 & 0.01 & 27.4 \\
    Mean Diff. & \tb{21.8}** & \tb{12.2}** & \tb{2.18}** & \tb{0.51}** & \tb{-4.4}** & \tb{12.4}** & \tb{11.5}** & \tb{0.58}** & 0.01 & -2.2 \\

    &  \multicolumn{5}{c}{\rule{0pt}{1.5em} \textbf{RWR}} & \multicolumn{5}{c}{\rule{0pt}{1.5em} \textbf{FB}} \\
    \cmidrule(r){2-6} \cmidrule(r){7-11} \\
    Original & 43.4 & 102.0 & 0.05 & 0.92 & 7.4 & 9.2 & 100.6 & 0.14 & 40.09 & 17.5 \\ 
    Autofocused & 68.5 & 113.4 & 1.34 & 0.63 & 14.2 & 63.4 & 110.8 & 0.63 & 0.49 & 11.2 \\
    Mean Diff. & \tb{25.1}** & \tb{11.5}** & \tb{1.30}** & \tb{-0.29}** & \tb{6.8}** & \tb{14.2}** & \tb{10.2}* & \tb{0.50}** & \tb{0.40}** & \tb{-6.3}** \\
    
    & \multicolumn{5}{c}{\rule{0pt}{1.5em}\textbf{CEM-PI}} & \multicolumn{5}{c}{\rule{0pt}{1.5em}\textbf{CMA-ES}} \\
    \cmidrule(r){2-6} \cmidrule(r){7-11} \\
    Original & 34.5 & 55.8 & 0.00 & -0.16 & 148.3 & 42.1 & 69.4 & 0.00 & 0.27 & 27493.2 \\
    Autofocused & 59.5 & 89.1 0.39 & & 0.02 & 46.9 & 45.1 & 70.7 & 0.00 & 0.50 & 27944.9 \\
    Mean Diff. & \tb{25.0}* & \tb{33.3}* & \tb{0.39}* & 0.18 & \tb{-101.4}** & 3.1 & 1.2 & 0 & \tb{0.22}* & 451.7 \\
    
    \bottomrule
  \end{tabular*}
  }
\ifthenelse{\boolean{ARXIV}}{}{\vspace{-0.4cm}}
\end{table}

\paragraph{Oracle capacity.} To see how different oracle capacities affect the improvements gained from autofocusing, we ran the same experiments as Table \ref{tab:stats} with two different oracle architectures. One architecture had higher capacity than the original oracle (hidden layer sizes of $\texttt{(200, 200, 100, 100, 10)}$ compared to $\texttt{(100, 100, 100, 100, 10)}$; Table \ref{tab:stats_big}), and one one had lower capacity (hidden layer sizes of $\texttt{(100, 100, 10)}$; Table \ref{tab:stats_small}).

\begin{table}
  \caption{Designing superconducting materials. Same experiments and caption as Table 1, except using an oracle architecture with hidden layers $\texttt{200} \rightarrow \texttt{200} \rightarrow \texttt{100} \rightarrow \texttt{100} \rightarrow \texttt{10}$. We ran six different MBO methods, each with and without autofocusing. For each method, we extracted those samples with oracle expectations above the $\sth{80}$ percentile and computed their ground-truth expectations. We report the median and maximum of those ground-truth expectations (both in degrees K), their percent chance of improvement (PCI, in percent) over the maximum label in the training data, as well as the Spearman correlation ($\rho$) and root mean squared error (RMSE, in degrees K) between the oracle and ground-truth expectations. Each reported score is averaged over $10$ trials, where, in each trial, a different training set was sampled from the training distribution. ``Mean Diff.'' is the average difference between the score when using autofocusing compared to not. Bold values with one star (*) and two stars (**), respectively, mean $p$-value $< 0.05$ and $< 0.01$ from a two-sided Wilcoxon signed-rank test on the $10$ paired score differences. For all scores but RMSE, a higher value means autofocusing yielded better results (as indicated by the arrow $\uparrow$); for RMSE, the opposite is true (as indicated by the arrow $\downarrow$). 
  }
  \label{tab:stats_big}
  \tiny{
  \begin{tabular*}{\textwidth}{l @{\extracolsep{\fill}} lllll lllll}
    \toprule
    & Median $\uparrow$ & Max $\uparrow$ & PCI $\uparrow$ & $\rho \uparrow$ & RMSE $\downarrow$ & Median $\uparrow$ & Max $\uparrow$ & PCI $\uparrow$ & $\rho \uparrow$ & RMSE $\downarrow$ \\
    & \multicolumn{5}{c}{\textbf{CbAS}} & \multicolumn{5}{c}{\textbf{DbAS}} \\
    \cmidrule(r){2-6} \cmidrule(r){7-11} \\
    
    Original & 48.3 & 100.8 & 0.05 & 0.03 & 19.6 & 55.3 & 98.6 &  0.025 & -0.02  & 32.1  \\
    Autofocused & 79.0 & 119.4 & 4.35 & 0.55 & 13.5 & 81.6 & 113.3 & 5.33 & 0.01 & 27.0 \\
    Mean Diff. & \tb{30.7}** & \tb{18.6}** & \tb{4.30}** & \tb{0.52}** & \tb{-6.1}** & \tb{26.4}** & \tb{14.8}** & \tb{5.30}** & 0.03 & -5.1 \\

    &  \multicolumn{5}{c}{\rule{0pt}{1.5em} \textbf{RWR}} & \multicolumn{5}{c}{\rule{0pt}{1.5em} \textbf{FB}} \\
    \cmidrule(r){2-6} \cmidrule(r){7-11} \\
    Original & 36.5 & 81.3 & 0.00 & -0.24 & 55.5 & 47.8 & 101.5 & 0.09 & 0.06 & 18.3 \\ 
    Autofocused & 73.4 & 114.8 & 2.05 & 0.72 & 12.7 & 63.5 & 113.1 & 0.58 & 0.58 & 10.7 \\
    Mean Diff. & \tb{36.9}** & \tb{33.4}** & \tb{2.05}** & \tb{0.97}** & \tb{-42.8}** & \tb{15.7}** & \tb{11.7}** & \tb{0.49}** & \tb{0.51}** & \tb{-7.5}** \\
    
    & \multicolumn{5}{c}{\rule{0pt}{1.5em}\textbf{CEM-PI}} & \multicolumn{5}{c}{\rule{0pt}{1.5em}\textbf{CMA-ES}} \\
    \cmidrule(r){2-6} \cmidrule(r){7-11} \\
    Original & 48.2 & 58.3 & 0.00 & 0.09 & 271.4 & 39.0 & 63.1 & 0.00 & 0.26 & 6774.6 \\
    Autofocused & 64.5 & 84.1 & 0.48 & -0.14 & 61.07 & 53.1 & 79.0 & 0.01 & 0.48 & 10183.7 \\
    Mean Diff. & 16.3 & \tb{25.9}* & 0.48 & -0.22 & \tb{-210.3} & \tb{14.1}* & \tb{15.9}* & 0.01 & 0.23 & 3409.1 \\
    
    \bottomrule
  \end{tabular*}
  }
\ifthenelse{\boolean{ARXIV}}{}{\vspace{-0.4cm}}
\end{table}

\begin{table}
  \caption{Designing superconducting materials. Same experiments and caption as Table 1, except using an oracle architecture with hidden layers $\texttt{100} \rightarrow \texttt{100} \rightarrow \texttt{10}$. We ran six different MBO methods, each with and without autofocusing. For each method, we extracted those samples with oracle expectations above the $\sth{80}$ percentile and computed their ground-truth expectations. We report the median and maximum of those ground-truth expectations (both in degrees K), their percent chance of improvement (PCI, in percent) over the maximum label in the training data, as well as the Spearman correlation ($\rho$) and root mean squared error (RMSE, in degrees K) between the oracle and ground-truth expectations. Each reported score is averaged over $10$ trials, where, in each trial, a different training set was sampled from the training distribution. ``Mean Diff.'' is the average difference between the score when using autofocusing compared to not. Bold values with one star (*) and two stars (**), respectively, mean $p$-value $< 0.05$ and $< 0.01$ from a two-sided Wilcoxon signed-rank test on the $10$ paired score differences. For all scores but RMSE, a higher value means autofocusing yielded better results (as indicated by the arrow $\uparrow$); for RMSE, the opposite is true (as indicated by the arrow $\downarrow$). 
  }
  
  \label{tab:stats_small}
  \tiny{
  \begin{tabular*}{\textwidth}{l @{\extracolsep{\fill}} lllll lllll}
    \toprule
    & Median $\uparrow$ & Max $\uparrow$ & PCI $\uparrow$ & $\rho \uparrow$ & RMSE $\downarrow$ & Median $\uparrow$ & Max $\uparrow$ & PCI $\uparrow$ & $\rho \uparrow$ & RMSE $\downarrow$ \\
    \cmidrule(r){2-6} \cmidrule(r){7-11} \\
    
    Original & 0.06 & 46.8 & 98.5 & -0.03 & 23.8 & 0.02 & 56.3 & 97.7 & 0.00 & 37.0 \\
    Autofocused & 1.4 & 67.0 & 114.3 & 0.52 & 13.0 & 1.3 & 72.5 & 108.4 & 0.04 & 27.6 \\
    Mean Diff. & \tb{1.3}** & \tb{20.2}** & \tb{15.8}** & \tb{0.55}** & \tb{-10.9}** & \tb{1.3}** & \tb{16.2}** & \tb{10.7}** & 0.03 & \tb{-9.4}** \\

    &  \multicolumn{5}{c}{\rule{0pt}{1.5em} \textbf{RWR}} & \multicolumn{5}{c}{\rule{0pt}{1.5em} \textbf{FB}} \\
    \cmidrule(r){2-6} \cmidrule(r){7-11} \\
    Original & 0.00 & 30.9 & 76.8 & -0.33 & 83.5 & 0.04 & 47.2 & 100.4 & 0.02 & 19.9 \\ 
    Autofocused & 0.68 & 66.0 & 112.6 & 0.57 & 18.3 & 0.43 & 58.2 & 111.4 & 0.50 & 12.3 \\
    Mean Diff. & \tb{0.68}** & \tb{35.1}** & \tb{35.8}** & \tb{0.90}** & \tb{-65.2}** & \tb{0.40}** & \tb{11.0}** & \tb{11.0}** & \tb{0.48}** & \tb{-7.6}** \\
    
    & \multicolumn{5}{c}{\rule{0pt}{1.5em}\textbf{CEM-PI}} & \multicolumn{5}{c}{\rule{0pt}{1.5em}\textbf{CMA-ES}} \\
    \cmidrule(r){2-6} \cmidrule(r){7-11} \\
    Original & 0.00 & 36.3 & 46.2 & -0.01 & 382.4 & 0.00 & 36.9 & 62.3 & 0.10 & 9587.1 \\
    Autofocused & 0.04 & 53.9 & 71.3 & -0.04 & 210.8 & 0.00 & 43.9 & 80.0 & 0.29 & 40858.6 \\
    Mean Diff. & 0.04 & 17.6 & \tb{25.1}* & -0.03 & -171.6 & 0 & \tb{7.0}* & \tb{17.7}* & \tb{0.19}** & \tb{31271.5}** \\
    
    \bottomrule
  \end{tabular*}
  }
\ifthenelse{\boolean{ARXIV}}{}{\vspace{-0.4cm}}
\end{table}

\end{document}